\documentclass{article}
\usepackage[accepted]{icml2024}

\usepackage{times}
\usepackage[utf8]{inputenc} %
\usepackage[T1]{fontenc}    %
\usepackage{url}            %
\usepackage{booktabs}       %
\usepackage{nicefrac}       %
\usepackage{microtype}      %
\usepackage{color}
\usepackage{graphicx}

\usepackage{enumitem}
\usepackage{amsmath}
\usepackage{amsfonts}
\usepackage{amssymb}
\usepackage{mathtools}

\usepackage{amsthm}

\newcommand{\Figref}[1]{Figure~\ref{#1}}  %
\newcommand{\Secref}[1]{Sec.~\ref{#1}} %
\newcommand{\secref}[1]{Sec.~\ref{#1}} %

\newcommand{\appref}[1]{Appendix~\ref{#1}}

\usepackage{xspace}
\makeatletter
\DeclareRobustCommand\onedot{\futurelet\@let@token\@onedot}
\def\@onedot{\ifx\@let@token.\else.\null\fi\xspace}
\makeatother

\newcommand{\eg}{e.g\onedot}
\newcommand{\ie}{i.e\onedot}

\usepackage{xr-hyper}
\makeatletter
\newcommand*{\addFileDependency}[1]{%
  \typeout{(#1)}
  \@addtofilelist{#1}
  \IfFileExists{#1}{}{\typeout{No file #1.}}
}
\makeatother

\usepackage{xcolor}
\definecolor{ourblue}{rgb}{0.368,0.507,0.71}
\definecolor{ourorange}{rgb}{0.881,0.611,0.142}
\definecolor{ourgreen}{rgb}{0.56,0.692,0.195}
\definecolor{ourred}{rgb}{0.923,0.386,0.209}
\definecolor{ourviolet}{rgb}{0.528,0.471,0.701}
\definecolor{ourbrown}{rgb}{0.772,0.432,0.102}
\definecolor{ourlightblue}{rgb}{0.364,0.619,0.782}
\definecolor{ourdarkgreen}{rgb}{0.572,0.586,0.}

\definecolor{gd_color}{rgb}{0.368,0.507,0.71}
\definecolor{lpgd_color}{rgb}{0.56,0.692,0.195}

\definecolor{ourcyan2}{rgb}{0.125,0.722,0.804}
\definecolor{ourred2}{rgb}{0.863,0.184,0.047}
\definecolor{ouryellow2}{cmyk}{0,0.16,1.0,0.07}
\definecolor{ourviolet2}{cmyk}{0.55,0.56,0,0.47}
\definecolor{ourorange2}{cmyk}{0,0.46,0.89,0.11}

\theoremstyle{plain}
\newtheorem{theorem}{Theorem}[section]
\newtheorem{proposition}[theorem]{Proposition}
\newtheorem{lemma}[theorem]{Lemma}

\theoremstyle{definition}
\newtheorem{example}{Example}

\theoremstyle{remark}

\graphicspath{{graphics/}}

\usepackage{amsmath,amsfonts,bm}

\def\Figref#1{Figure~\ref{#1}}

\def\secref#1{section~\ref{#1}}
\def\Secref#1{Section~\ref{#1}}

\def\eqref#1{(\ref{#1})}
\def\twoeqrefs#1#2{(\ref{#1},\,\ref{#2})}
\def\threeeqrefs#1#2#3{(\ref{#1},\,\ref{#2},\,\ref{#3})}

\def\Algref#1{Algorithm~\ref{#1}}

\def\Lemmaref#1{Lemma~\ref{#1}}
\def\Thmref#1{Theorem~\ref{#1}}

\def\1{\bm{1}}

\DeclareMathAlphabet{\mathsfit}{\encodingdefault}{\sfdefault}{m}{sl}
\SetMathAlphabet{\mathsfit}{bold}{\encodingdefault}{\sfdefault}{bx}{n}

\def\sR{{\mathbb{R}}}

\newcommand{\R}{\mathbb{R}}

\DeclareMathOperator{\Solver}{SolverOracle}

\usepackage{fnpct} %
\usepackage{hyperref}
\usepackage{url}
\hypersetup{
	colorlinks=true,       %
	linkcolor=gd_color,        %
	citecolor=lpgd_color,        %
	filecolor=gd_color,     %
	urlcolor=lpgd_color
}

\usepackage{caption,subcaption}
\usepackage{algorithm}
\renewcommand{\algorithmiccomment}[1]{{\color{lpgd_color}// #1}}

\usepackage{amssymb}

\newcommand{\lag}{\ensuremath{\mathcal{L}}}

\newcommand{\cone}{\mathcal{K}}
\newcommand{\define}{\ensuremath{\coloneqq}}

\newcommand{\X}{\mathcal{X}}
\newcommand{\Y}{\mathcal{Y}}

\newcommand{\D}{\mathcal{D}}

\newcommand{\dfw}{\Delta}
\newcommand{\dbw}{\mathrm{d}}

\newcommand{\x}{x}

\newcommand{\s}{s}
\newcommand{\y}{y}

\newcommand{\z}{z}

\newcommand{\w}{w}
\renewcommand{\u}{u}
\renewcommand{\v}{v}
\newcommand{\f}{f}
\newcommand{\g}{g}
\newcommand{\h}{h}
\renewcommand{\c}{c}
\renewcommand{\b}{b}

\newcommand{\backbone}{W}

\newcommand{\inv}[1]{\tfrac{1}{#1}}

\newcommand{\Q}{Q}
\renewcommand{\H}{H}
\renewcommand{\L}{L}

\newcommand{\inp}{\mu}

\newcommand\mathmiddlescript[1]{\vcenter{\hbox{$\scriptstyle #1$}}}

\newcommand{\lpgd}{LPGD\xspace}
\newcommand{\lpgda}{LPGD$\mathmiddlescript\tau$\xspace}
\newcommand{\lpgdl}{LPGD$_\tau$\xspace}
\newcommand{\lpgdu}{LPGD$^\tau$\xspace}
\newcommand{\lppm}{LPPM\xspace}
\newcommand{\lppma}{LPPM$\mathmiddlescript\tau$\xspace}
\newcommand{\lppml}{LPPM$_\tau$\xspace}
\newcommand{\lppmu}{LPPM$^\tau$\xspace}

\newcommand{\lmd}{LMD\xspace}

\DeclareMathOperator*{\argmin}{arg\,min}
\DeclareMathOperator*{\argmax}{arg\,max}

\DeclareMathOperator*{\diag}{diag}

\DeclareMathOperator{\prox}{prox}

\DeclareMathOperator{\env}{env}

\icmltitlerunning{LPGD: A General Framework for Backpropagation through Embedded Optimization Layers}

\begin{document}

\twocolumn[
\icmltitle{LPGD: A General Framework for Backpropagation \\through Embedded Optimization Layers}

\begin{icmlauthorlist}
\icmlauthor{Anselm Paulus}{unitue,mpi}
\icmlauthor{Georg Martius}{unitue,mpi}
\icmlauthor{V\'it Musil}{unimas}
\end{icmlauthorlist}

\icmlaffiliation{mpi}{Max Planck Institute for Intelligent Systems, T\"ubingen, Germany}
\icmlaffiliation{unitue}{University of Tübingen, T\"ubingen, Germany}
\icmlaffiliation{unimas}{Masaryk University, Brno, Czech Republic}

\icmlcorrespondingauthor{Anselm Paulus}{apaulus@tue.mpg.de}
\icmlcorrespondingauthor{Georg Martius}{gmartius@uni-tuebingen.de}
\icmlcorrespondingauthor{V\'it Musil}{musil@fi.muni.cz}

\icmlkeywords{Machine Learning, ICML, bilevel optimization, differentiable optimization}

\vskip 0.3in
]

\printAffiliationsAndNotice{} %

\begin{abstract}
    Embedding parameterized optimization problems as layers into machine learning architectures serves as a powerful inductive bias. 
    Training such architectures with stochastic gradient descent requires care, as degenerate derivatives of the embedded optimization problem often render the gradients  uninformative.
    We propose \mbox{\emph{Lagrangian Proximal Gradient Descent} (LPGD)} a flexible framework for training architectures with embedded optimization layers that seamlessly integrates into automatic differentiation libraries.
    \lpgd efficiently computes meaningful replacements of the degenerate optimization layer derivatives by re-running the forward solver oracle on a perturbed input.
    \lpgd captures various previously proposed methods as special cases, while fostering deep links to traditional optimization methods. 
    We theoretically analyze our method and demonstrate on historical and synthetic data that \lpgd converges faster than gradient descent even in a differentiable setup.
    
\end{abstract}

\section{Introduction}\label{sec:introduction}
Optimization at inference is inherent to many prediction tasks, including autonomous driving \citep{paden2016survey}, modeling physical systems \citep{cramer2020lagrangian}, or robotic control \citep{kumar2016optimal}.
Therefore, embedding optimization algorithms as building blocks of machine learning models serves as a powerful inductive bias.
A~recent trend has been to embed parameterized constrained optimization problems that can efficiently be solved to optimality~\citep{amos2017optnet, agrawal2019differentiable, agrawal2019differentiating, VlastelicaEtal2020, sun2022alternating, sahoo2022gradient}.

Training such a \emph{parameterized} optimization model is an instance of bi-level optimization \citep{gould2016differentiating}, which is generally challenging.
Whenever it is possible to propagate gradients through the optimization problem via an informative derivative of the solution mapping, the task is typically approached with standard stochastic gradient descent~(GD) \citep{amos2017optnet, agrawal2019differentiating}.
However, when the optimization problem has discrete solutions, the derivatives are typically degenerate, as small perturbations of the input do not affect the optimal solution. Previous works have proposed several methods to overcome this challenge, ranging from differentiable relaxations \citep{wang2019satnet, Wilder2019:melding, MandiGuns2020:InteriorPointPO, djolonga2017differentiable} and stochastic smoothing \citep{berthet2020learning, dalle2022learning}, over proxy losses \citep{paulus2021:comboptnet}, to finite-difference based techniques \citep{VlastelicaEtal2020}.\looseness=-1

The main contribution of this work is the unification of a variety of previous methods \citep{mcallester2020direct, VlastelicaEtal2020, domke2010implicit, sahoo2022gradient, elmachtoub2022smart, blondel2020learning} into a general framework called \emph{Lagrangian Proximal Gradient Descent}~(\lpgd).
Motivated by traditional proximal optimization techniques \citep{moreau1962fonctions, rockafellar1970convex, nesterov1983method, figueiredo2007gradient, tseng2008accelerated, beck2009fast, combettes2011proximal, bauschke2011convex, nesterov2014introductory, parikh2014proximal}, we derive \lpgd as gradient descent on a smoothed envelope of a loss linearization. This fosters deep links between traditional and contemporary methods.
We provide theoretical insights into the asymptotic behavior of our method, capturing a trade-off between smoothness and tightness of the introduced envelope.

We identify multiple practical use-cases of \lpgd.
On the one hand, when non-degenerate derivatives of the solution mapping exist, they can be computed as the limit of the \lpgd update, providing a fast and simple alternative to previous methods based on differentiating the optimality conditions \citep{amos2017optnet, agrawal2019differentiating, Wilder2019:melding, MandiGuns2020:InteriorPointPO}.
On the other hand, when the derivatives are degenerate and GD fails, \lpgd still allows learning the optimization parameters.
This generalizes \citet{VlastelicaEtal2020} to non-linear objectives, saddle-point problems, and learnable constraint parameters.
Finally, we explore a new experimental direction by demonstrating on synthetic and historical data that \lpgd can result in faster convergence than GD even when non-degenerate derivatives of the solution mapping exist.

\section{Related work}\label{sec:related-work}
Numerous implicit layers have been proposed in recent years, including neural ODEs \citep{chen2018neural, dupont2019augmented} and root-solving-based layers \citep{bai2019deep, bai2020multiscale, gu2020implicit, winston2020monotone, fung2021fixed, elghaoui2021implicit, geng2021training}.
In this work, we focus on optimization-based layers. A lot of research has been done on obtaining the gradient of such a layer, either by using the implicit function theorem to differentiate quadratic programs \citep{amos2017optnet}, conic programs \citep{agrawal2019differentiating}, ADMM \citep{sun2022alternating}, dynamic time warping \citep{xu2023deep}, or by finite-differences \citep{domke2010implicit, mcallester2020direct, song2016training, lorberbom2019direct}.

Another direction of related work has investigated optimization problems with degenerate derivatives of the solution mapping. The techniques developed for training these models range from continuous relaxations of SAT problems \citep{wang2019satnet} and submodular optimization \citep{djolonga2017differentiable}, over regularization of linear programs \citep{amos2019limited, Wilder2019:melding, MandiGuns2020:InteriorPointPO, paulus2020gradient} to stochastic smoothing \citep{berthet2020learning, dalle2022learning}, learnable proxies \citep{wilder2019end} and generalized straight-through-estimators \citep{jang2016categorical, sahoo2022gradient}. Other works have built on geometric proxy losses \citep{paulus2021:comboptnet} and, again, finite-differences \citep{VlastelicaEtal2020, niepert2021implicit, minervini2023adaptive}.

Finally, a special case of an optimization layer is to embed an optimization algorithm as the final component of the prediction pipeline. This encompasses energy-based models \citep{lecun2005energy, blondel2022energy}, structured prediction \citep{mcallester2020direct, blondel2019structured, blondel2020learning}, smart predict-then-optimize \citep{ferber2020mipaal, elmachtoub2022smart} and symbolic methods such as SMT solvers \citep{fredrikson2023learning}.
Additional details of the closest related methods are in~\appref{sec:extended-related-work}.

\section{Problem Setup}\label{sec:problem-setup}
We consider a parameterized embedded constrained optimization problem of the form
\begin{align}\label{eq:optimal-lagrangian}
\begin{split}
	\lag^*(\w)
	\define &\min_{\x\in\X} \max_{\y\in\Y} \lag(\x, \y, \w)
\end{split}
\end{align}
where $\w\in\sR^k$ are the parameters, $\X\subseteq\sR^n$ and $\Y\subseteq\sR^m$ are the primal and dual feasible set, and $\lag\in\mathcal{C}^1$ is a \emph{Lagrangian}. 
The corresponding optimal solution is
\begin{equation}\label{eq:embedded-opt-problem}
	\z^*(\w)
		= 
        (\x^*, \y^*)(\w)
		\define \arg\min_{\x\in\X} \max_{\y\in\Y} \lag(\x, \y, \w).
\end{equation}
We assume strong duality holds for \eqref{eq:optimal-lagrangian}.
For instance, this setup covers conic programs and quadratic programs, see~\appref{sec:implicit-function-theorem} for details.
Note, that the solution of~\eqref{eq:embedded-opt-problem} is in general set-valued. 
We assume that the solution set is non-empty and has a selection $\z^*(\w)$ continuous at $\w$.\!%
\footnote{These assumptions \eg follow from compactness of $\X$ and $\Y$ and the existence of a unique solution, given the continuity of $\lag$.}
Throughout the paper, we assume access to an oracle that efficiently solves~\eqref{eq:embedded-opt-problem} to high accuracy. 
In our experiments, \eqref{eq:embedded-opt-problem} is a conic program that we efficiently solve to high accuracy using the SCS solver~\citep{odonoghue2016conic}.\!%
\footnote{We use CVXPY~\citep{diamond2016cvxpy, agrawal2019differentiable} for automatic reduction of parameterized convex optimization problems to conic programs in a differentiable way.}

Our aim is to embed optimization problem~\eqref{eq:embedded-opt-problem} into a larger prediction pipeline. Given an input $\mu\in\sR^p$ (\eg an image), the parameters of the embedded optimization problem $\w$ are predicted by a parameterized backbone model $\backbone_\theta\colon \sR^p\rightarrow \sR^k$ (\eg a neural network with weights $\theta\in\sR^r$) as $\w=\backbone_\theta(\mu)$. The embedded optimization problem~\eqref{eq:embedded-opt-problem} is then solved on the predicted parameters $\w$ returning the predicted solution $\x^*(\w)$, and its quality is measured by a loss function $\ell\colon \sR^n \rightarrow \sR$. The backbone and the loss function are assumed to be continuously differentiable.

Our goal is to train the prediction pipeline by minimizing the loss on a dataset of inputs $\{\mu_i\}_{i=1}^N$
\begin{align}\label{eq:outer-problem}
    \min_{\theta\in\sR^r} \textstyle\sum_{i=1}^N\ell\bigl(\x^*(\backbone_\theta(\mu_i))\bigr)
\end{align}
using gradient backpropagation as in stochastic gradient descent or variations thereof~\citep{kingma2015adam}.
However, the solution mapping does not need to be differentiable, and even when it is, the derivatives are often degenerate (\eg they can be zero almost everywhere).\!%
\footnote{Our method only assumes continuity of the solution mapping which is weaker than differentiability. Therefore, whenever the true gradients exist, the continuity assumption is also~fulfilled.}
Therefore, we aim to derive informative replacements for the gradient $\nabla_\w\ell(\x^*(\w))$, which can then be further backpropagated to the weights $\theta$ by standard automatic differentiation libraries~\citep{abadi2015tensorflow, bradbury2018jax, paszke2019pytorch}.
Note that the loss could also be composed of further learnable components that might be trained simultaneously.
A list of symbols is provided in~\appref{sec:list-of-symbols}.

\section{Backgound: Proximal Point Method \& Proximal Gradient Descent}\label{proximal-gradient-descent}

The \emph{Moreau envelope}~\citep{moreau1962fonctions} $\env_{\tau\!\f}\colon \sR^n \rightarrow \sR$ of a proper, lower semi-continuous, possibly non-smooth 
function $\f\colon\sR^n\to\sR$ is defined for $\tau>0$~as
\begin{equation}\label{eq:moreau-envelope}
	\env_{\tau\!\f}(\widehat\x)
		\define \inf\nolimits_{\x} \f(\x) + 
	\inv{2} \|\x - \widehat\x\|_2^2.
\end{equation}
The envelope $\env_{\tau\!\f}$ is a smoothed lower bound approximation of $\f$ \citep[Theorem~1.25]{rockafellar1998variational}.
The corresponding \emph{proximal map} $\prox_{\tau\!\f}\colon \sR^n \to \sR^n$ is given~by
\begin{align}\label{eq:proximal-map}
	\begin{split}
    	\prox_{\tau\!\f}(\widehat\x)
    	   &\define \arg\min\nolimits_{\x} \f(\x) + \inv{2\tau} \|\x - \widehat\x\|_2^2
        \\ &\phantom{:}
    		= \widehat\x - \tau\nabla \env_{\tau\!\f}(\widehat\x)
	\end{split}
\end{align}
and can be interpreted as a gradient descent step on the Moreau envelope with step-size $\tau$.
For a more detailed discussion of the connection between proximal map and Moreau envelope see  \eg~\citet{rockafellar1998variational, bauschke2011convex, parikh2014proximal}.

The \emph{proximal point method}~\citep{rockafellar1976monotone, guler1992new, bauschke2011convex, parikh2014proximal}
aims to minimize $\f$ by iteratively updating $\widehat\x\mapsto\prox_{\tau\!\f}(\widehat\x)$.
Now, assume that $\f$ decomposes as $\f = \g + \h$, with $\g$ differentiable and $\h$ potentially non-smooth,
and consider a linearization of $\g$ around $\widehat\x$ given by
\begin{align}\label{eq:linear-loss-approx}
    \widetilde{\g}(\x) \define \g(\widehat\x) + \langle \x-\widehat\x,\nabla \g(\widehat \x)\rangle.
\end{align}
The corresponding proximal map reads as
\begin{align}
	   \prox_{\tau(\widetilde\g+\h)}(\widehat\x) 
   		&= \arg\min\nolimits_{\x} \widetilde\g(\x) + \h(\x) +  \inv{2\tau} \|\x - \widehat\x\|_2^2\nonumber
			\\ &
	   	= \prox_{\tau\h}\bigl(\widehat\x-\tau\nabla\g(\widehat\x)\bigr)\label{eq:proximal-gradient-descent}
\end{align}
and iterating $\widehat\x\mapsto\prox_{\tau\h}\bigl(\widehat\x-\tau\nabla\g(\widehat\x)\bigr)$ is called \emph{proximal gradient descent}~\citep{nesterov1983method, combettes2011proximal, parikh2014proximal}.

\section{Method}\label{sec:method}

Our goal is to translate the idea of proximal methods to parameterized optimization models as in \Secref{sec:problem-setup}, by defining a \emph{Lagrange-Moreau envelope} of the loss $\w\mapsto\ell(\x^*(\w))$ on which we can perform gradient descent.
In analogy to~\eqref{eq:moreau-envelope}, given $\w$ and corresponding optimal solution $\x^*$, the envelope should select an $\x$ in the proximity of $\x^*$ with a lower loss~$\ell$.
The key concept is to replace as a measure of proximity the Euclidean distance with a \emph{Lagrangian divergence} indicating how close $\x$ is to optimality given~$\w$.

\subsection{Lagrangian Divergence}

First we define the \emph{Lagrangian difference} for $\x\in\X$ and $\w\in\sR^k$ as
\begin{align} \label{eq:lagrangian-difference}
    D_\lag(\x, \y | \w) 
        \define  \lag (\x, \y, \w) - \lag^*(\w),
\end{align}
where $\lag^*(\w)$ is the optimal Lagrangian~\eqref{eq:optimal-lagrangian}.
We then define the \emph{Lagrangian divergence}%
\footnote{%
	In some cases, the Lagrangian divergence coincides with the \emph{Bregman divergence}, which generalizes the squared Euclidean distance, opening connections to \emph{Mirror descent}, see~\appref{sec:mirror-descent}.}
$D^*_\lag(\cdot | \w)\colon \X\rightarrow \sR^+$ as
\begin{align}
\begin{split}\label{eq:lagrangian-divergence}
    D^*_\lag(\x | \w) 
    &\define \sup_{\y\in\Y}D_\lag(\x, \y | \w) 
    \\&
    \phantom{:}= \sup_{\y\in\Y}\bigl[\lag (\x, \y, \w) - \lag^*(\w)\bigr]
    \geq 0
\end{split}
\end{align}
for $\w\in\sR^k$,
where the last inequality follows from
\begin{align}
    \sup_{\y\in\Y} \lag(\x, \y, \w) \geq \min_{\widetilde\x\in\X} \max _{\y\in\Y} \lag(\widetilde\x, \y, \w) = \lag^*(\w).
\end{align}
The divergence has the key property
\begin{align}\label{eq:lagrangian-divergence-property}
    D^*_\lag(\x | \w) = 0
		\quad\Leftrightarrow\quad
	\text{$\x\in\X$ minimizes \eqref{eq:embedded-opt-problem}}
\end{align}
which makes it a reasonable measure of optimality of $\x$ given~$\w$, 
for the proof, see~\appref{sec:proofs}.

\subsection{Lagrange-Moreau Envelope}

Given $\tau>0$, we say that $\ell_{\tau}\colon\sR^k\to\sR$ is the \emph{lower Lagrange-Moreau envelope} ($\lag$-envelope) if
\begin{align}\label{eq:lower-lag-moreau-envelope}
	\begin{split}
    	\ell_{\tau}(\w)
        &\define \min_{\x\in\X} \ell(\x) + \inv{\tau} D^*_\lag(\x | \w)
            \\&
        \phantom{:}= \min_{\x\in\X} \max_{\y\in\Y} \ell(\x) + \inv{\tau} D_\lag(\x, \y| \w).
	\end{split}
\end{align}
The corresponding \emph{lower Lagrangian proximal map} $\z_{\tau}\colon \sR^k \to\sR^{n+m}$ is defined as
\begin{align}\label{eq:lower-lag-proximal-map}
	\begin{split}
    	\z_{\tau}(\w)
        &\define \argmin_{\x\in\X} \max_{\y\in\Y} \ell(\x) + \inv{\tau}D_\lag(\x, \y| \w)
            \\&
        \phantom{:}= \arg\min_{\x\in\X} \max_{\y\in\Y} \lag(\x, \y, \w) + \tau\mspace{-1mu}\ell(\x).
	\end{split}
\end{align}
The \textit{upper $\lag$-envelope}
$\ell^{\tau}\colon\sR^k\to\sR$ is defined analogously with maximization instead of minimization as
\begin{align}\label{eq:upper-lag-moreau-envelope}
	\begin{split}
    	\ell^{\tau}(\w)
        &\define \max_{\x\in\X} \ell(\x) - \inv{\tau}D^*_\lag(\x | \w)
            \\&
        \phantom{:}= \max_{\x\in\X} \min_{\y\in\Y} \ell(\x) - \inv{\tau}D_{\lag}(\x, \y | \w),
	\end{split}
\end{align}
and the corresponding \emph{upper $\lag$-proximal map} $\z^{\tau}\colon \sR^k \rightarrow \sR^{n+m}$ ($\lag$-proximal map) as
\begin{align}\label{eq:upper-lag-proximal-map}
	\begin{split}
   	\z^{\tau}(\w)
    	&\define \arg\max_{\x\in\X} \min_{\y\in\Y} \ell(\x) - \inv{\tau} D_{\lag}(\x,\y | \w)
            \\&
    	\phantom{:}= \arg\min_{\x\in\X} \max_{\y\in\Y} \lag(\x, \y, \w) - \tau\mspace{-1mu}\ell(\x).
	\end{split}
\end{align}
The lower and upper $\lag$-envelope are lower and upper bound approximations of the loss~$\w\mapsto\ell(\x^*(\w))$, respectively.
We emphasize that the solutions to~\eqref{eq:lower-lag-moreau-envelope} and~\eqref{eq:upper-lag-moreau-envelope} are in general set-valued and we assume that they are non-empty and admit a single-valued selection that is continuous at $\w$, which we denote by~\eqref{eq:lower-lag-proximal-map} and~\eqref{eq:upper-lag-proximal-map}. We also assume that strong duality holds for the upper and lower $\lag$-envelopes.\!%
\footnote{In general, these assumptions cast strong restrictions on the allowed loss functions. This is one of the reasons why our main focus will be on loss approximations introduced in \secref{sec:lpgd}.}
We will also work with the \emph{average} $\lag$-envelope
\begin{align} \label{eq:eq:average-lag-proximal-map}
	\ell\mathmiddlescript\tau(\w) 
  	&\define \inv{2}\bigl[\ell_{\tau}(\w) + \ell^{\tau}(\w)\bigr].
\end{align}
The different envelopes are closely related to right-, left- and double-sided directional derivatives.

\subsection{Lagrangian Proximal Point Method}

Our goal will be to perform gradient descent on the $\lag$-envelope \twoeqrefs{eq:lower-lag-moreau-envelope}{eq:upper-lag-moreau-envelope}.
The gradients of the $\lag$-envelopes read
\begin{align}
\begin{split}
    	\nabla \ell_{\tau}(\w) 
        &= \,\inv{\tau}\nabla_\w\bigl[\lag(\z_{\tau}, \w) - \lag(\z^*, \w)\bigr],
            \\
    	\nabla \ell^{\tau}(\w) 
        &= \inv{\tau}\nabla_\w\bigl[\lag(\z^*, \w) - \lag(\z^{\tau}, \w)\bigr],
        \label{eq:lppm}
\end{split}
\end{align}
where we abbreviate $\z_{\tau}=\z_{\tau}(\w)$ and $\z^{\tau}=\z^{\tau}(\w)$.
The proof is in~\appref{sec:proofs}.
In analogy to the proximal point method~\eqref{eq:proximal-map} we refer to GD using  \eqref{eq:lppm} as the \emph{Lagrangian Proximal Point Method} (\lppm), or specifically, \lppml, \lppma and \lppmu for GD on $\ell_\tau$, $\ell\mathmiddlescript\tau$ and~$\ell^\tau$, respectively. 
\looseness=-1

\begin{example}[Direct Loss Minimization]
For an input $\inp\in\sR^p$, label $\x_\text{true}\in\X$, loss $\ell\colon \X \times \X \rightarrow \sR$, feature map 
$\Psi\colon \X \times \sR^p \rightarrow \sR^k$
and an optimization problem of the form
\begin{align}
    \x^*(\w, \inp)=\argmin_{\x\in\X} -\langle \w, \Psi(\x, \inp) \rangle
\end{align}
the \lppml update \eqref{eq:lppm} reads
\begin{align}
    \nabla \ell_\tau(\w) = \inv{\tau}\bigl[\Psi(\x^*, \inp) - \Psi(\x_{\tau}, \inp)\bigr],
\end{align}
with $\x^*=\x^*(\w, \inp)$ and	
\begin{align}
    \x_{\tau}
    = \argmin_{\x\in\X} -\langle \w, \Psi(\x, \inp) \rangle + \tau \ell(\x, \x_\text{true}).
\end{align}
This recovers the ``towards-better'' \emph{Direct Loss Minimization} (DLM) update \citep{mcallester2020direct}, while the ``away-from-worse'' update corresponds to the \lppmu update, both of which were proposed in the context of taking the limit $\tau\rightarrow 0$, which aims to compute the true gradients.
\end{example}

\subsection{Lagrangian Proximal Gradient Descent}\label{sec:lpgd}

\lppm requires computing the $\lag$-proximal map~\eqref{eq:lower-lag-proximal-map} or~\eqref{eq:upper-lag-proximal-map}. 
Due to the loss term, efficiently solving the involved optimization problem might not be possible or it requires choosing and implementing a custom optimization algorithm that is potentially much slower than the oracle used to solve the forward problem~\eqref{eq:embedded-opt-problem}.
Instead, we aim to introduce an approximation of the loss that allows solving the $\lag$-proximal map with the forward solver oracle.
We first observe that in many cases the parameterized Lagrangian takes the form
\begin{align}\label{eq:lag-decomposition}
    \lag(\x, \y, \w) = \langle \x, \c \rangle + \Omega(\x, \y, \v),
\end{align}
with $\w=(\c,\v)$, linear parameters~$\c\in\sR^n$, non-linear parameters~$\v\in\sR^{k-n}$ and continuously differentiable~$\Omega$.
Our approximation, inspired by proximal gradient descent~\eqref{eq:proximal-gradient-descent}, is to consider a linearization $\smash{\widetilde\ell}$ of the loss $\ell$ at $\x^*$.\!%
\footnote{Note that other approximations of the loss can also be used depending on the parameterization supported by the forward solver. For example, if quadratic terms are supported we could consider a quadratic loss approximation.}
Importantly, the loss linearization is only applied \emph{after the solver} and does not approximate or linearize the solution mapping.
Abbreviating $\nabla \ell=\nabla \ell(\x^*)$, we get the \lag-proximal~maps\looseness=-1
\begin{align}
     \widetilde\z_{\tau}(\w)
     & \define \arg\min_{\x\in\X} \max_{\y\in\Y} \langle \x, \c \rangle + \Omega(\x, \y, \v) + \tau\langle \x, \nabla\ell\rangle \nonumber
            \\
    & \phantom{:}= \z^*(\c + \tau\nabla\ell, \v),
     \label{eq:lower-lag-proximal-map-linearized}
        \\
	 \widetilde\z^{\mspace{2mu}\tau\mspace{-2mu}}(\w)
     & \define \arg\min_{\x\in\X} \max_{\y\in\Y} \langle \x, \c \rangle + \Omega(\x, \y, \v) - \tau\langle \x, \nabla\ell\rangle \nonumber
        \\
    & \phantom{:}= \z^*(\c - \tau\nabla\ell, \v).
     \label{eq:upper-lag-proximal-map-linearized}
\end{align}
As these are instances of the forward problem~\eqref{eq:embedded-opt-problem} on different parameters, they can be computed with the same solver oracle.
The assumed strong duality and continuity of the solution of~\eqref{eq:embedded-opt-problem} also typically implies the same properties for the perturbed problems \twoeqrefs{eq:lower-lag-proximal-map-linearized}{eq:upper-lag-proximal-map-linearized}.
Note that warm-starting the solver with $\z^*$ can strongly accelerate the computation, often making the evaluation of the $\lag$-proximal map much faster than the forward problem.
This enables efficient computation of the $\lag$-envelope gradient~\eqref{eq:lppm} as
\begin{align}
\begin{split}
	\nabla \widetilde\ell_{\tau}(\w) 
    &= \,\inv{\tau}
					\nabla_\w\bigl[
							\lag\bigl(\w, \widetilde\z_{\tau}\bigr)
						- \lag\bigl(\w, \z^*\bigr)
					\bigr],
        \\
	\nabla \widetilde\ell^{\tau}(\w) 
    &= \inv{\tau}
					\nabla_\w\bigl[
							\lag\bigl(\w, \z^*\bigr)
						- \lag\bigl(\w, \widetilde\z^{\mspace{2mu}\tau\mspace{-2mu}}\bigr)
					\bigr].
     \label{eq:lpgd}
\end{split}
\end{align}
In analogy to proximal gradient descent~\eqref{eq:proximal-gradient-descent} we refer to gradient descent using \eqref{eq:lpgd} as \emph{Lagrangian Proximal Gradient Descent} (\lpgd), or more specifically, to \lpgdl, \lpgda and~\lpgdu for GD on $\widetilde\ell_\tau$, $\widetilde\ell\mathmiddlescript\tau$ and~$\widetilde\ell^\tau$, respectively.  

\lpgd can also be viewed as computing gradients of the loss $\nabla_\w\ell(\x^*(\w))$ standardly by rolling out the chain rule, but replacing every appearance of the optimization layer co-derivative with a certain finite-difference.\!%
\footnote{Note that this only requires a single additional solver evaluation, in contrast to traditional finite-difference gradient estimation requiring $k$ solver evaluations, which is often intractable.}
This shows that \lpgd only affects the optimization layer, providing a derivative replacement that carries higher-order information than linear sensitivities.
Further, this viewpoint directly implies that \lpgd smoothly integrates into existing automatic differentiation frameworks~\citep{abadi2015tensorflow, bradbury2018jax, paszke2019pytorch}, by simply replacing the backward pass operation of the optimization layer with the finite-difference, as summarized for \lpgdl in \Algref{alg:lpgd-lower}.
\begin{algorithm}[t]
\begin{algorithmic}
       \FUNCTION{ForwardPass($\w$)}
       \STATE $\displaystyle\z^* \gets \Solver(\w)$ \hfill\COMMENT{Solve optimization~\eqref{eq:embedded-opt-problem}}
       \STATE\textbf{save} $\w$, $\z^*$ for backward pass
       \STATE\textbf{return} $\z^*$
       \ENDFUNCTION
       \medskip
       \FUNCTION{BackwardPass($\nabla\ell=\nabla_\z\ell(\z^*)$, $\tau$)}
       \STATE\textbf{load} $(\c,\v)=\w$ and $\z^*$ from forward pass
       \STATE $\w_\tau \gets \c+\tau\nabla\ell, \v$ \hfill\COMMENT{Perturb parameters}
       \STATE $\displaystyle\widetilde\z_{\tau} \gets \Solver(\w_\tau)$\hfill\COMMENT{\eqref{eq:lower-lag-proximal-map-linearized}, warmstart~with~$\z^*$}
       \STATE $\nabla_\w \widetilde\ell_{\tau}(\w) = \inv{\tau}\nabla_\w\bigl[\lag\bigl(\widetilde\z_{\tau}, \w\bigr) - \lag\bigl(\z^*, \w\bigr)\bigr]$%
       \hfill\COMMENT{\eqref{eq:lpgd}}
       \STATE\textbf{return} $\nabla_\w\widetilde\ell_{\tau}(\w)$\hfill\COMMENT{Gradient of $\lag$-envelope}
       \ENDFUNCTION
\end{algorithmic}
\caption{Forward and Backward Pass of \lpgdl}
\label{alg:lpgd-lower}
\end{algorithm}

\begin{example}[Blackbox Backpropagation]
For a linear program (LP)%
\footnote{For an LP over a polytope $\X$ the space of possible solutions is discrete. Whenever the solution is unique, which is true for almost every $\w$, the solution mapping is locally constant (and hence continuous) around $\w$. Therefore our continuity assumptions hold for almost all $\w$.}
\begin{align}\label{eq:linear-program}
    \x^*(\c)=\argmin_{\x\in\X} \langle \x, \c \rangle,
\end{align}
the \lpgdl update \eqref{eq:lpgd} reads
\begin{align}
    \nabla \widetilde\ell_\tau(\c) 
    = \inv{\tau}\bigl[\widetilde\x_\tau(\c) - \x^*(\c)\bigr]
    = \inv{\tau}\bigl[\x^*(\c+\tau\nabla\ell) - \x^*(\c)\bigr],\nonumber
\end{align}
which recovers the update rule in \emph{Blackbox Backpropagation} (BB) \citep{VlastelicaEtal2020}.
The piecewise affine interpolation of the loss $\c \mapsto \smash{\widetilde\ell}(\x^*(\c))$ derived in BB agrees with the lower $\lag$-envelope~$\widetilde\ell_\tau$.
\looseness=-1
\end{example}

\begin{example}[Implicit Differentiation by Perturbation]
For a regularized linear program
\begin{align}\label{eq:regularized-lp}
    \x^*(\c)=\argmin_{\x\in\X} \langle \x, \c \rangle + \Omega(\x)
\end{align}
with a strongly convex regularizer $\Omega\colon \X\rightarrow \R$, the \lpgda update \eqref{eq:lpgd} reads
\begin{align}
\begin{split}
    \nabla \widetilde\ell\mathmiddlescript\tau(\c) 
    &= \inv{2\tau}\bigl[\widetilde\x_\tau(\c) - \widetilde\x^{\mspace{2mu}\tau\mspace{-2mu}}(\c)\bigr]
        \\&
    = \inv{2\tau}\bigl[\x^*(\c+\tau\nabla\ell) - \x^*(\c-\tau\nabla\ell)\big],
\end{split}
\end{align}
recovering the update in~\citet{domke2010implicit}, where only the limit case $\tau\rightarrow 0$ is considered. 
\end{example}

\subsection{Regularization \& Augmented Lagrangian}\label{sec:augmented-lagrangian}
To increase the smoothness of the $\lag$-envelope, we augment the Lagrangian with a strongly convex regularizer
\begin{align}\label{eq:augmentation}
    \lag_\rho(\x, \y, \w) &\define \lag(\x, \y, \w) + \inv{2\rho} \|\x - \x^*\|_2^2
\end{align}
with $\rho>0$.
Equivalently, we may re-introduce the quadratic regularizer from the Moreau envelope~\eqref{eq:moreau-envelope} into the \lag-envelope~\eqref{eq:lower-lag-moreau-envelope} and \lag-proximal map~\eqref{eq:lower-lag-proximal-map}%
\begin{align}
\begin{split}
	\ell_{\tau\rho}(\w)
    & \define \min_{\x\in\X} \max_{\y\in\Y} \ell(\x)
        + \inv{\tau}D_\lag(\x, \y | \w)
        \\
    & \phantom{\define \min_{\x\in\X} \max_{\y\in\Y} \ell(\x)}\quad
        + \inv{2\rho} \|\x - \x^*\|_2^2,
    \label{eq:augmented-lower-lagrange-moreau-envelope}
\end{split}
    \\
\begin{split}
	\z_{\tau\rho}(\w)
    & \define \arg\min_{\x\in\X} \max_{\y\in\Y} \ell(\x) 
        + \inv{\tau}D_\lag(\x, \y | \w)
        \\
    & \phantom{\define \arg\min_{\x\in\X} \max_{\y\in\Y} \ell(\x)}\quad
        + \inv{2\rho} \|\x - \x^*\|_2^2.
    \label{eq:augmented-lower-lagrangian-proximal-map}
\end{split}
\end{align}
These definitions have an analogy for the upper envelope and for a linearized loss, which we omit for brevity.
The \lppml and \lpgdl updates then take the form
\begin{align}
\begin{split}
    \nabla \ell_{\tau\rho}(\w)
    &= \inv\tau\nabla_\w\bigl[
            \lag\bigl(\w, \z_{\tau\rho}\bigr)
            - \lag\bigl(\w, \z^*\bigr)
        \bigr],
        \\
    \nabla \widetilde\ell_{\tau\rho}(\w)
    &= \inv\tau\nabla_\w\bigl[
            \lag\bigl(\w, \widetilde\z_{\tau\rho}\bigr)
            - \lag\bigl(\w, \z^*\bigr)
        \bigr].
    \label{eq:augmented-lower-lppm-lpgd}
\end{split}
\end{align}
The augmentation does not alter the current optimal solution, but smoothens the Lagrange-Moreau envelope.
This also has connections to Jacobian-regularization in the implicit function theorem, which we discuss in~\appref{sec:implicit-function-theorem}.
Note that using this quadratic regularization with \lpgd requires the solver oracle to support quadratic objectives, which is the case for the conic program solver used in our experiments.
We visualize the smoothing of the different $\lag$-envelopes of the loss $\c\mapsto\ell(\x^*(\c))$ in \Figref{fig:envelope-visualization-rho}, for a quadratic loss on the solution to the linear program \eqref{eq:linear-program} with $\X=[0,1]^n$ and a one-dimensional random cut through the cost space.\!%
\footnote{Here, the envelopes appear non-continuous due to the loss linearization.
Namely, the point $\x^*(\c)$ at which we take the loss linearization varies with $\c$, and therefore the function of which we take the envelope varies as well.}
\begin{figure}[tb]
    \centering
    \includegraphics[width=\linewidth]{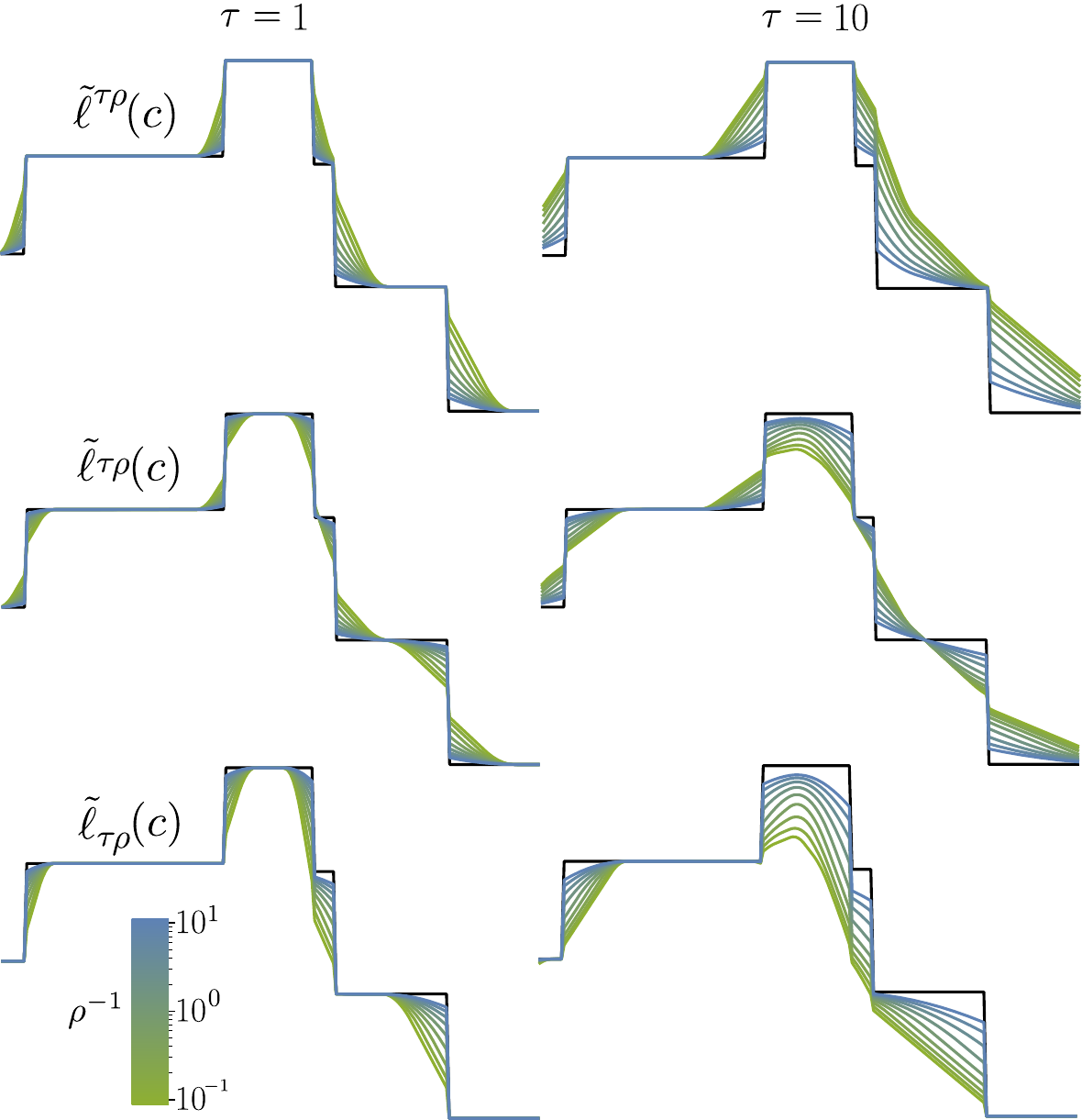}
    \caption{Visualization of the upper $\widetilde\ell^{\tau\rho}$, average $\widetilde\ell\mathmiddlescript{\tau\rho}$, and lower $\widetilde\ell_{\tau\rho}$ Lagrange-Moreau envelope for different temperatures $\tau$ and augmentation strengths $\rho$.
    The envelopes are smoothed approximations of the loss $\c\mapsto \ell(\x^*(\c))$ (black).
    \looseness=-1
    }
    \label{fig:envelope-visualization-rho}
\end{figure}

\section{Method Analysis}\label{sec:asymptotic-behavior}
We now theoretically analyze our method. Proofs of the statements are given in \appref{sec:proofs}.
The following proposition characterizes the continuity properties that the $\lag$-envelope inherits from the Lagrangian.
\begin{proposition}\label{prop:lipschitz}
Assume that $\lag$ is $L$-Lipschitz continuous in $\w$. Then $\ell_\tau,\ell\mathmiddlescript\tau,\ell^\tau$ are $\tfrac{2L}{\tau}$-Lipschitz continuous in $\w$.
\end{proposition}
This shows that by increasing $\tau$ we get smaller Lipschitz bounds of the $\lag$-envelope, giving a form of ``smoothening''.

Next, we characterize the asymptotic behavior of the $\lag$-envelope, $\lag$-proximal map, \lppm and \lpgd updates.
First, we consider the limit as $\tau\rightarrow 0$.
Let $X^*(\w)$ and $\widehat\X(\w)$ denote the \emph{optimal solution set} and the \emph{effective feasible set}, respectively, defined as
\begin{align}
    X^*(\w)
        &\define \{\x\in\X \mid \D^*_\lag(\x |\w) = 0\},\label{eq:effective-feasible-set}\\
    \widehat\X(\w)
        &\define \{\x\in\X \mid \D^*_\lag(\x |\w) < \infty\}.\label{eq:optimal-solution-set}
\end{align}
\begin{proposition}\label{prop:tau-to-zero}
    Let $w$ be such that $X^*(w)\ne\emptyset$.
    Assume $\lag,\ell$ are lower semi-continuous and $\ell$ is finite-valued on $\X$.
    Then it holds that
    \begin{align}\label{eq:tau-to-zero-envelope}
        \lim_{\tau\rightarrow 0}\ell_\tau(\w)
        \define \min_{\x^*\in{X^*(\w)}} \ell(\x^*)
    \end{align}
    and also
    \begin{align}
        \lim_{\tau\rightarrow 0}\x^*_\tau(\w)
        \in \argmin_{\x^*\in{X^*(\w)}} \ell(\x^*).
    \end{align}
    whenever the limit exists.
\end{proposition}
Propositions~\ref{prop:lipschitz} and \ref{prop:tau-to-zero} highlight that choosing $\tau$ determines a trade-off between smoothness and tightness of the approximation.
Next, we show that the \lpgd update converges to the true gradient.
\looseness=-1

\begin{theorem}\label{thm:limit-gradient}
Let $\lag\in\mathcal{C}^2$ and assume the optimizer of~\eqref{eq:embedded-opt-problem} admits a differentiable selection $\x^*(\w)$ at~$\w$. Then
\begin{align}
    \lim_{\tau\rightarrow 0} \nabla  \widetilde\ell_{\tau}(\w) 
    = \nabla_{\w} \ell(\x^*(\w))
    = \lim_{\tau\rightarrow 0} \nabla  \widetilde\ell^{\tau}(\w).
\end{align}
\end{theorem}

The proof, also highlighting the connections between \lpgd updates and finite-difference approximations of forward-mode derivatives, is given in \appref{sec:proofs}.  
Theorem \ref{thm:limit-gradient} asserts that \lpgd computes the same gradients in the limit as methods based on the implicit function theorem, such as \emph{OptNet} \citep{amos2017optnet}, regularized LPs \citep{Wilder2019:melding, MandiGuns2020:InteriorPointPO}, or differentiable conic programs \citep{agrawal2019differentiating}.\looseness=-1

Next, we consider the limit $\tau\rightarrow\infty$. 
First, we have the result for the lower \lag-proximal map~\eqref{eq:lower-lag-proximal-map}.

\begin{proposition}\label{prop:tau-to-infty}
    Let $w$ be such that $\widehat\X(w)\ne\emptyset$. Then
    \begin{align}\label{eq:lower-lppm-arg-limit}
        \lim_{\tau\rightarrow\infty}\x_\tau(\w)
        \in \argmin_{\x\in{\widehat\X(\w)}} \ell(\x)
    \end{align}
    whenever the limit exists. For a linearized loss, we have
    \begin{align}\label{eq:lower-lpgd-arg-limit}
        \lim_{\tau\rightarrow\infty}\widetilde\x_\tau(\w)
        \in \argmin_{\x\in{\widehat\X(\w)}} \langle \x, \nabla\ell\rangle
        = \x_{FW}(\w),
    \end{align}
    where $\x_{FW}$ is the solution to a Frank-Wolfe iteration LP~\citep{frank1956algorithm}
\end{proposition}
Next proposition covers the case of the lower $\lag$-proximal map \eqref{eq:augmented-lower-lagrangian-proximal-map} with a quadratic regularizer.

\begin{proposition}\label{prop:tau-to-infty-augmented}
The primal lower $\lag$-proximal map \eqref{eq:augmented-lower-lagrangian-proximal-map} turns into the standard proximal map \eqref{eq:proximal-map}
\begin{align}\label{eq:augmented-lower-lppm-arg-limit}
\begin{split}
    \lim_{\tau\rightarrow\infty}\x_{\tau\rho}(\w)
    &= \argmin_{\x\in{\widehat\X(\w)}} \bigl[\ell(\x) + \inv{2\rho} \|\x - \x^*\|_2^2\bigr]
        \\&
    = \prox_{\rho\ell + I_{\widehat\X(\w)}}(\x^*),
\end{split}
\end{align}
whenever the limit exists. For a linearized loss, it reduces to the Euclidean projection onto $\widehat\X(\w)$
\begin{align}\label{eq:augmented-lower-lpgd-arg-limit}
\begin{split}
    \lim_{\tau\rightarrow\infty}\widetilde\x_{\tau\rho}(\w)
    &= \argmin_{\x\in{\widehat\X(\w)}} \bigl[\langle \x, \nabla \ell\rangle + \inv{2\rho}\|\x - \x^*\|_2^2\bigr]
        \\&
    = P_{\widehat\X(\w)}(\x^* - \rho \nabla\ell).
\end{split}
\end{align}
\end{proposition}
The \lppml \twoeqrefs{eq:lppm}{eq:augmented-lower-lppm-lpgd} and \lpgdl \twoeqrefs{eq:lpgd}{eq:augmented-lower-lppm-lpgd} updates corresponding to the \lag-proximal maps \twoeqrefs{eq:lower-lppm-arg-limit}{eq:augmented-lower-lppm-arg-limit} and \twoeqrefs{eq:lower-lpgd-arg-limit}{eq:augmented-lower-lpgd-arg-limit} have the interpretation of decoupling the update step, by first computing a ``target'' (\eg $\widetilde\x_{\tau\rho}$ via projected gradient descent with step-size $\rho$), and then minimizing the Lagrangian divergence to make the target the new optimal solution.

We discuss multiple examples that showcase the asymptotic variations of \lppm and \lpgd.
Here, we will work with the finite-difference version of the updates~\twoeqrefs{eq:lppm}{eq:augmented-lower-lppm-lpgd}, denoted~by
\begin{align}
    \Delta \ell_{\tau}(\w)
    &\define \tau\nabla\ell_{\tau}(\w),
        &
    \Delta \ell^{\tau}(\w)
    &\define \tau\nabla\ell^{\tau}(\w),\label{eq:finite-differences-lppm}
        \\
    \Delta \ell_{\tau\rho}(\w)
    &\define \tau\nabla\ell_{\tau\rho}(\w),
        &
    \Delta \widetilde\ell_{\tau\rho}(\w)
    &\define \tau\nabla\widetilde\ell_{\tau\rho}(\w).\label{eq:finite-differences-augmented}
\end{align}

\begin{example}[Identity with Projection]
For an LP \eqref{eq:linear-program}
it is $\widehat\X(\w) = \X$ and we get the asymptotic regularized \lpgdl update \eqref{eq:augmented-lower-lppm-lpgd} in finite-difference form \eqref{eq:finite-differences-augmented} as
\begin{align}\label{eq:augmented-lower-lpgd-limit}
\begin{split}
    \lim_{\tau\rightarrow\infty} \Delta \widetilde\ell_{\tau\rho}(\c) 
    & = \lim_{\tau\rightarrow\infty}\bigl[\widetilde\x_{\tau\rho}(\c) - \x^*\bigr]
        \\&
    = P_{\X}(\x^*-\rho\nabla\ell) - \x^*,
\end{split}
\end{align}
where we used \eqref{eq:augmented-lower-lpgd-arg-limit}.
In the limit of large regularization $\rho\rightarrow 0$ with division by $\rho$ in analogy to Theorem \ref{thm:limit-gradient},
the above update converges to
\begin{align*}
\begin{split}
    \lim_{\rho\rightarrow 0}\lim_{\tau\rightarrow\infty} \inv{\rho}\Delta \widetilde\ell_{\tau\rho}(\c) 
    = \lim_{\rho\rightarrow 0} \inv{\rho} \bigl[P_{\X}(\x^*-\rho\nabla\ell) - \x^*\bigr]&
        \\
    = DP_{\X}(\x^*|-\nabla\ell)
    = D^*P_{\X}(\x^*|-\nabla\ell)&,
\end{split}
\end{align*}
where $DP$ and $D^*P$ denote the directional derivative and coderivative of the projection  $P$ at~$\x^*$.
This is closely related to the \textit{Identity with Projection} method by \citet{sahoo2022gradient}, in which the true gradient is replaced by backpropagating~$-\nabla\ell$ through the projection onto a relaxation of~$\X$.\!%
\footnote{Note that this also has close ties to the one-step gradient arising in implicit differentiation of fixed-point iterations by treating the inverse Jacobian as an identity function \citep{geng2021training, chang2022object, bai2022deep}.}
\end{example}

\begin{example}[Smart Predict then Optimize]
The \emph{Smart Predict then Optimize} (SPO) setting \citep{mandi2020smart, elmachtoub2022smart} embeds an LP \eqref{eq:linear-program}
as the final component of the prediction pipeline and assumes access to the ground truth cost $\c_\text{true}$. The goal is to optimize the 
SPO loss $\ell_\text{SPO}(\x^*(\c), \c_\text{true})=\langle \x^*(\c) - \x^*(\c_\text{true}), \c_\text{true} \rangle$. Due to the discreteness of the LP, the SPO loss has degenerate gradients with respect to $\c$, \ie they are zero almost everywhere and undefined otherwise.
Choosing $\tau=\inv{2}$ for the upper $\lag$-proximal map~\eqref{eq:upper-lag-proximal-map}, we get
\begin{align}
        \x^{\frac12}(\c) 
        &= \argmax_{\x\in\X}\langle \x - \x^*(\c_\text{true}), \c_\text{true}\rangle -  2\langle \x - \x^*,  \c \rangle\nonumber
        \\&
        = \argmax_{\x\in\X} \langle \x,  \c_\text{true} - 2\c \rangle
\end{align}
which gives the lower and upper \lppm %
updates \eqref{eq:lppm} in finite-difference form \eqref{eq:finite-differences-lppm}
\begin{align}
	\Delta \ell_{\tau}(\c)  
        = \x_{\tau}(\c) - \x^*
        \,\,\text{and}\,\,
	\Delta \ell^{\frac1{2}}(\c) 
        = \x^* - \x^{\frac1{2}}(\c).
\end{align}
Summing both the updates and taking the limit $\tau\rightarrow\infty$ yields the combined \lppm update
\begin{align}\label{eq:lppm-spo}
\begin{split}
	\lim_{\tau\rightarrow\infty} \bigl[\Delta \ell_{\tau}(\c) + \Delta \ell^{\frac1{2}}(\c)\bigr]
    = \lim_{\tau\rightarrow\infty} \bigl[\x_{\tau}(\c) - \x^{\frac1{2}}(\c)\big] &
        \\
    =\x^*(\c_\text{true}) - \x^{\frac1{2}}(\c)
    = \inv{2}\nabla\ell_\text{SPO+}(\c, \c_\text{true})&,
\end{split}
\end{align}
where we used \eqref{eq:lower-lppm-arg-limit}.
Note that as the SPO loss is already linear in $\x$, \lppm and \lpgd are equivalent.
Update \eqref{eq:lppm-spo} recovers the gradient of the SPO$+$ loss
\begin{align}
\begin{split}
    \ell_\text{SPO+}(\c, \c_\text{true}) 
    \define \sup_{\x\in\X} \langle \x,  \c_\text{true} - 2\c\rangle + 2\langle \x^*(\c_\text{true}), \c \rangle &
        \\
    - \langle \x^*(\c_\text{true}), \c_\text{true} \rangle&
\end{split}
\end{align}
introduced by \citet{elmachtoub2022smart}, which has found widespread applications.
\end{example}

\begin{example}[Fenchel-Young Losses\footnote{
Note that an analogous derivation holds for \emph{generalized} Fenchel-Young losses \citep{blondel2022energy}, in which the regularized LP is replaced with a regularized energy function.
}]
In the \emph{structured prediction} setting we consider the regularized LP \eqref{eq:regularized-lp}
as the final component of the prediction pipeline and assume access to the ground truth solutions $\x_\text{true}$. The goal is to bring $\x^*(\c)$ close to $\x_\text{true}$ 
by minimizing any loss $\ell(\x)$ that is minimized over $\X$ at $\x_\text{true}$.
We compute the asymptotic \lppml update \eqref{eq:lppm} in finite-difference form \eqref{eq:finite-differences-lppm} as
\begin{align}
\begin{split}
	\lim_{\tau\rightarrow\infty} \Delta \ell_{\tau}(\c)
    & = \lim_{\tau\rightarrow\infty} \bigl[\x_{\tau}(\c) - \x^*\bigr]
        \\&
    =\x_\text{true} - \x^*
    = \nabla\ell_\text{FY}(\c, \x_\text{true}),
\end{split}
\end{align}
where we used \eqref{eq:lower-lppm-arg-limit} to compute the limit.
This recovers the gradient of the Fenchel-Young loss\footnote{Note that \citet{blondel2020learning} consider maximization.}
\begin{align*}
    \begin{split}
    \ell_\text{FY}(\c, \x_\text{true}) 
    &\define  \max_{\x\in\X} \bigl[-\langle \c, \x\rangle \!-\! \Omega(\x)\bigr]  \!+\! \Omega(\x_\text{true}) \!+\! \langle \c, \x_\text{true}\rangle 
        \\
    &\phantom{:}=  \langle \c, \x_\text{true}\rangle + \Omega(\x_\text{true}) - \min_{\x\in\X} \bigl[\langle \c, \x\rangle + \Omega(\x)\bigr]
    \end{split}
\end{align*}
defined by \citet{blondel2020learning}.
Depending on the regularizer $\Omega$ and the feasible region $\X$, Fenchel-Young losses cover multiple structured prediction setups, including the \emph{structured hinge} \citep{tsochantaridis2005large}, \emph{CRF} \citep{lafferty2001conditional}, and \emph{SparseMAP} \citep{niculae2018sparsemap} losses.
\looseness=-1
\end{example}

Finally, note that solvers in general return approximate solutions, yielding an approximate $\lag$-envelop denoted by $\widehat\ell_\tau(\w)$, which we quantify in the following proposition.

\begin{proposition}\label{prop:approximation-error}
    Assume that the solvers for \twoeqrefs{eq:embedded-opt-problem}{eq:lower-lag-proximal-map} return $\varepsilon$-accurate solutions with $\delta$-accurate objective values. Then
    \begin{align}
        |\widehat\ell_\tau(\w) - \ell_\tau(\w)| \leq \tfrac{2\delta}{\tau}.
    \end{align}
    Moreover, if $\nabla_\w\lag(\x,\y,\w)$ is $L$-Lipschitz continuous in $(\x,\y)\in\X\times\Y$, it holds that
    \begin{align}
        \|\nabla\widehat\ell_\tau(\w) - \nabla\ell_\tau(\w)\| \leq \tfrac{2L\varepsilon}{\tau}.
    \end{align}
\end{proposition}

This shows that increasing the temperature $\tau$ helps reducing the error introduced by approximate solutions.

\section{Experiments}\label{sec:experiments}
Previous works have successfully applied variations of \lpgd and \lppm to two types of experimental settings:
1)~Producing informative gradient replacements when the derivative of the solution-mapping is degenerate~\citep{rolinek2020cvpr, rolinek2020deep, mandi2020smart, sahoo2022gradient, ferber2023surco} and
2)~efficiently computing the true gradient in the limit as $\tau\rightarrow 0$ when the derivative of the solution-mapping is non-degenerate~\citep{domke2010implicit, mcallester2020direct}.
However, our interpretation of \lpgd, as updates capturing higher-order sensitivities than standard gradients, suggests the application to a third potential setting:
3)~Using \lpgd with finite values of $\tau$ even when the derivative of the solution-mapping is non-degenerate. 
We will compare \lpgd to GD in two such cases:
a) Learning the rules of Sudoku from synthetic data and
b) tuning the parameters of a Markowitz control policy on historical trading data.

\paragraph{Implementation.}
We build on the CVXPY ecosystem~\citep{diamond2016cvxpy, agrawal2018rewriting, agrawal2019differentiable} to implement \lpgd for a large class of parameterized optimization problems. 
CVXPY automatically reduces parameterized optimization problems to conic programs in a differentiable way, which are then solved with the SCS solver~\citep{odonoghue2016conic}.
These solutions can then be differentiated based on the results of~\citet{agrawal2019differentiating}, which is natively implemented in CVXPY.\!%
\footnote{We modify this implementation to support the regularization term as described in \appref{sec:implicit-function-theorem}.}
\footnote{Note that for many optimization problems the automatic reduction results in conic programs that have degenerate derivatives, in which case the method of \citet{agrawal2019differentiating} returns a heuristic quantity for the gradient without strong theoretical support.}
As an alternative to the true conic program derivative computation based on~\citet{agrawal2019differentiating}, we implement Algorithm~\ref{alg:lpgd-lower} (in all variations).
This allows using \lpgd for the large class of parameterized convex optimziation problems supported by CVXPY without modification.
The code is available at
\href{https://github.com/martius-lab/diffcp-lpgd}{github.com/martius-lab/diffcp-lpgd}.

\subsection{Learning the Rules of Sudoku}\label{sec:experiments-sudoku}

We consider a version of the Sudoku experiment proposed by \citet{amos2017optnet}.
The task is to learn the rules of Sudoku in the form of linear programming constraints from pairs of incomplete and solved Sudoku puzzles.
See \appref{sec:experimental-details} for details.

The results are reported in Figure \ref{fig:sudoku-train-curves}. 
\lpgd reaches a lower final loss than GD, which suggests that \lpgda produces better update steps than standard gradients. 
We observe faster convergence of \lpgd compared to GD in terms of wallclock time, which is due to the faster backward pass computation resulting from warmstarting.
Note that the forward solution time over training increases as the initial random optimization problem becomes more structured as the loss decreases.
Additional results can be found in \appref{sec:experimental-details-sudoku}.

\begin{figure}[!htb]
    \centering
        \smallskip
        \includegraphics[width=\linewidth]{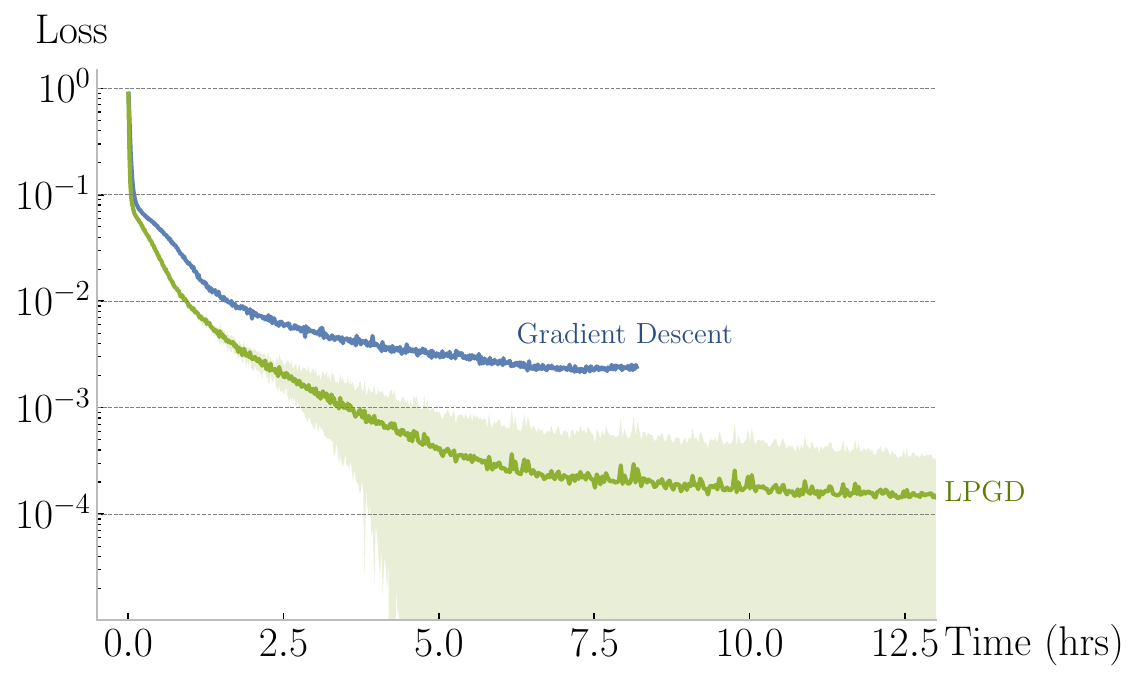}
        \smallskip
        \includegraphics[width=\linewidth]{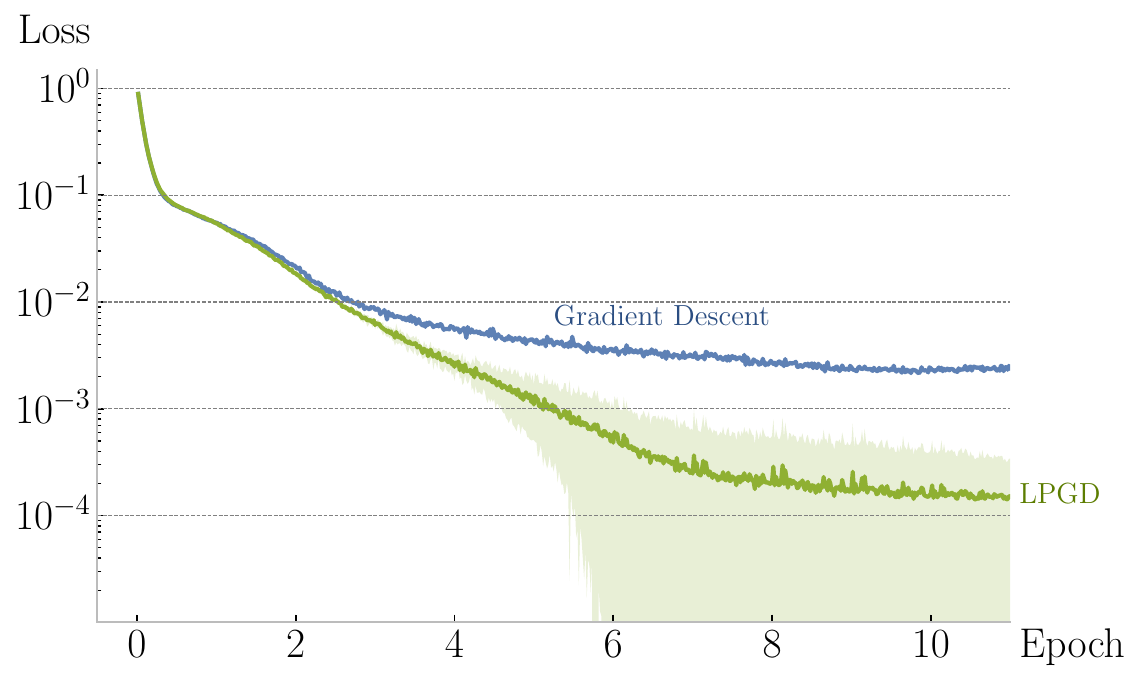}
        \smallskip
        \includegraphics[width=\linewidth]{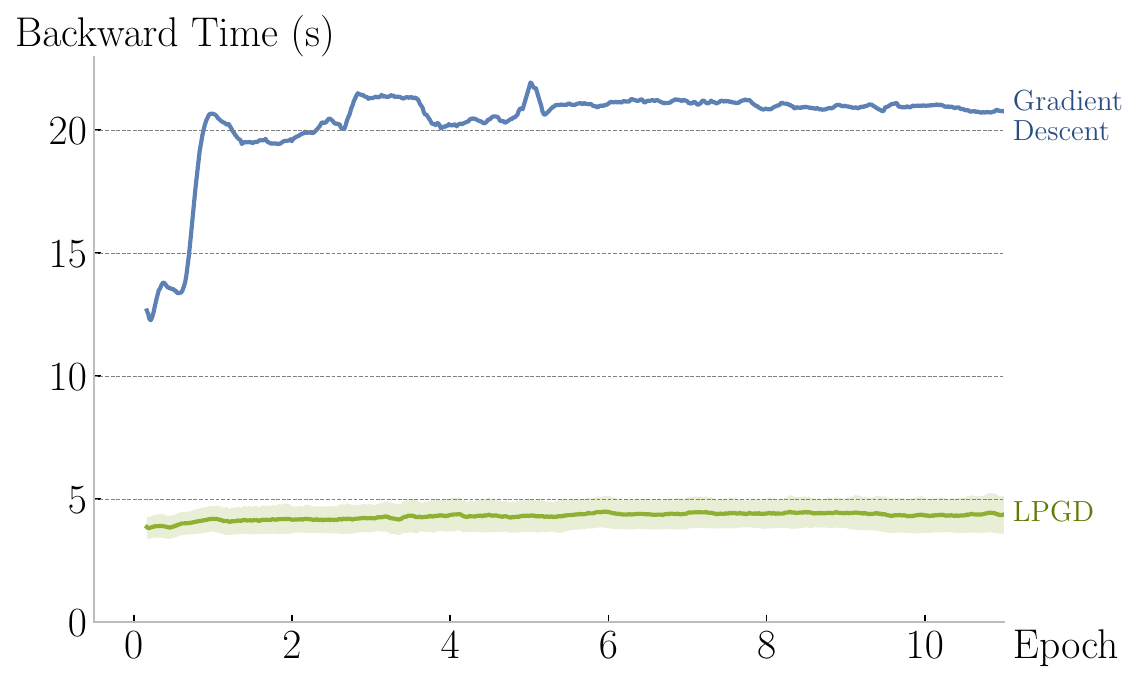}
        \smallskip
        \includegraphics[width=\linewidth]{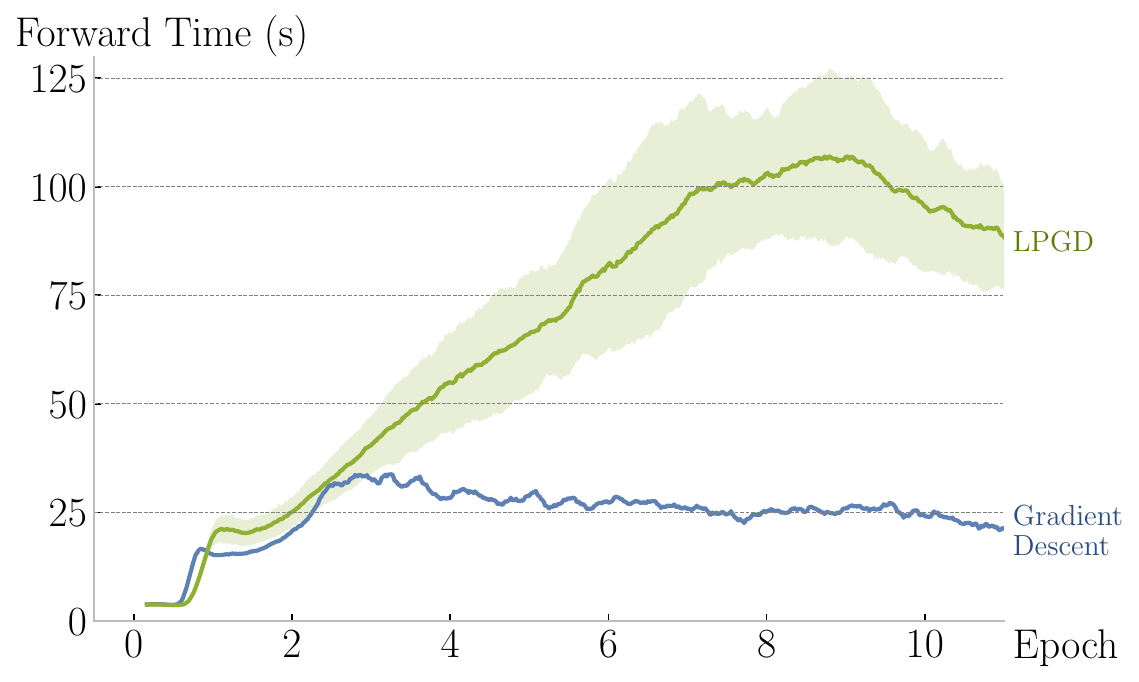}
    \caption{Comparison of \lpgd and gradient descent (GD) on the Sudoku experiment. 
    Reported train MSE over epochs, wall-clock time, and time spent in the backward and forward passes.
    Statistics are over $5$~restarts. 
    }
    \vspace{-2em}
    \label{fig:sudoku-train-curves}
\end{figure}

\subsection{Tuning a Markowitz Control Policy}\label{sec:experiments-portfolio}
We now consider the Markowitz Portfolio Optimization setting described by \citet[\S~5]{agrawal2020learning-control}.
The task is to tune a convex optimization control policy that iteratively trades assets over a trading horizon.
The policy parameters are initialized as a Markowitz model based on historical data.
Differentiating through the optimization problem in the control policy allows tuning them to maximize the utility on simulated evolutions of the asset values.
See \appref{sec:experimental-details-portfolio} for a more detailed description.

In~\Figref{fig:portfolio:lr}, we report a sweep over the learning rate $\alpha$.
The best-performing GD run achieves an improvement of $23\%$ over the Markowitz initialization, while the best-performing \lpgd run achieves an improvement of $35\%$.
Moreover, \lpgd converges faster in terms of iterations, with both methods requiring similar runtime.
Note that with a higher learning rate, GD becomes noisier than \lpgd.
For both methods, runs with $\alpha=0.1$ diverged and terminated due to infeasible conic programs.
In~\Figref{fig:portfolio:temperature}, we report a sweep on the temperature $\tau$ of~\lpgd.
For low $\tau=0.1$, \lpgd performs badly, as numerical issues paired with the backpropagation through time make the training dynamics unstable.
For $\tau=1$, \lpgd matches the performance of GD, as \lpgd updates approximate the true gradients.
The best performance is achieved for $\tau=100$, for which \lpgd provides higher-order information than the gradients.
For $\tau=1000$, the strong perturbations sometimes result in the infeasibility of the corresponding conic program, which causes the run to terminate.
This shows the practical trade-off in selecting the temperature parameter $\tau$, as discussed in the theoretical analysis of \lpgd.
In~\Figref{fig:portfolio:solver_acc}, we also conduct a sweep over the solver accuracy $\epsilon$, showing that GD and \lpgd have similar sensitivities to inaccurate solutions.
Additional results can be found in \appref{sec:experimental-details-portfolio}.

\begin{figure}[!t]
    \begin{subfigure}[c]{\linewidth}
        \includegraphics[width=\linewidth]{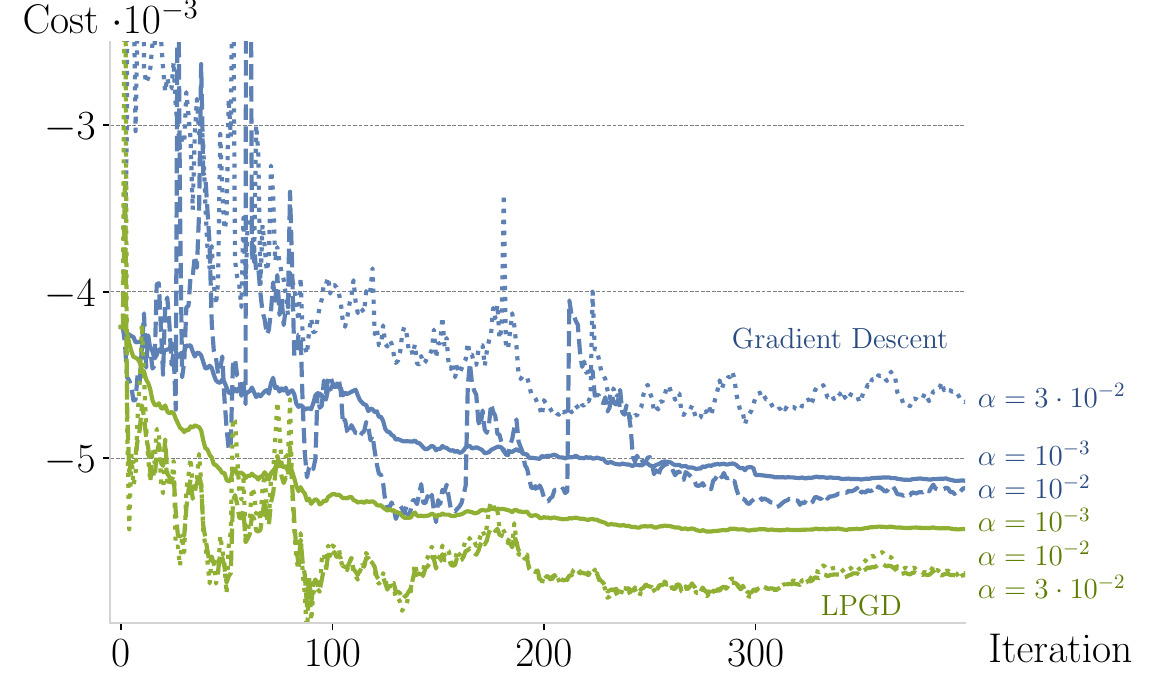}
        \caption{Learning rate $\alpha$ sweep.}
        \label{fig:portfolio:lr}
    \end{subfigure}
    \begin{subfigure}[c]{\linewidth}
        \includegraphics[width=\linewidth]{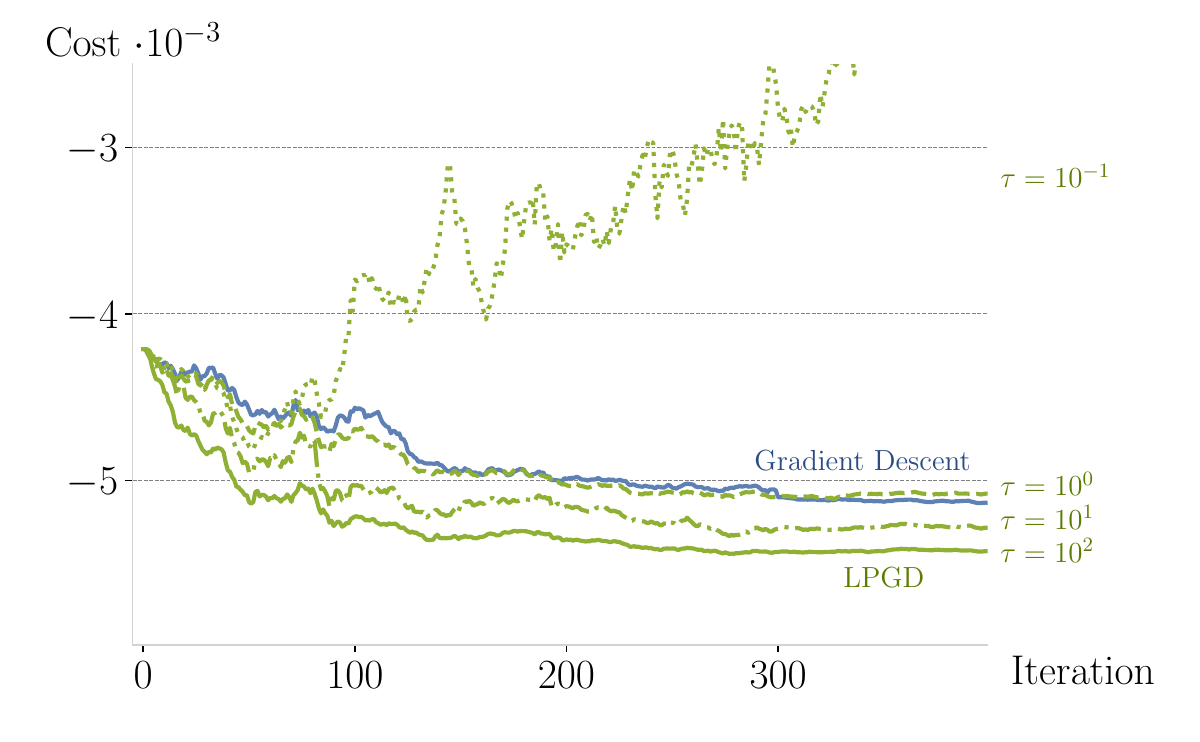}
        \caption{Temperature $\tau$ sweep.}
        \label{fig:portfolio:temperature}
    \end{subfigure}
    \begin{subfigure}[c]{\linewidth}
        \includegraphics[width=\linewidth]{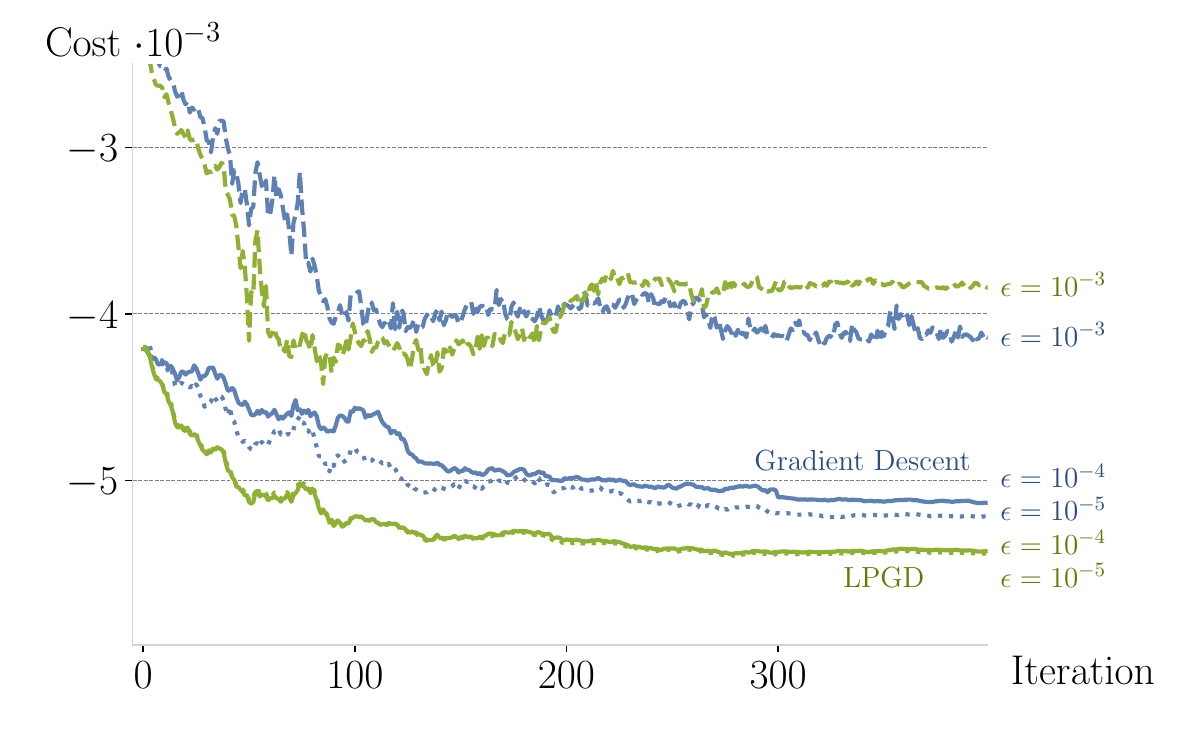}
        \caption{Solver accuracy $\epsilon$ sweep.}
        \label{fig:portfolio:solver_acc}
    \end{subfigure}
    \caption{
    Comparison of \lpgda and GD in the Markowitz portfolio optimization experiment. 
    When not specified otherwise, we use $\alpha=0.001$, $\epsilon=0.0001$, $\tau=100$ and $\rho=0$.
    }
    \label{fig:portfolio}
\end{figure}

\section{Conclusion}\label{sec:conclusion}
We propose \emph{Lagrangian Proximal Gradient Descent} (\lpgd), a flexible framework for learning parameterized optimization models.
\textbf{\lpgd~unifies and generalizes various state-of-the-art contemporary optimization methods}, including \emph{Direct Loss Minimization}~\citep{mcallester2020direct}, \emph{Blackbox Backpropagation}~\citep{VlastelicaEtal2020}, \emph{Implicit Differentiation by Perturbation}~\citep{domke2010implicit}, \emph{Identity with Projection} \citep{sahoo2022gradient}, \emph{Smart Predict then Optimize}~\citep{elmachtoub2022smart}, and \emph{Fenchel-Young losses}~\citep{blondel2020learning, blondel2022energy}, and \textbf{provides deep links to traditional optimization methods}.

\lpgd computes updates as finite differences and only requires accessing the forward solver as a black-box oracle, which makes it extremely simple to implement.
We also provide an implementation of \lpgd that smoothly integrates it into the CVXPY ecosystem.
Formulated as gradient descent on a loss function envelope,
\lpgd allows learning general objective and constraint parameters of saddlepoint problems even for solution mappings with degenerate derivatives.

Various special cases of \lpgd have shown impressive results in optimizing parameters of solution mappings with degenerate derivatives and speeding up the computation of non-degenerate derivatives.
We explore a new experimental direction by using \lpgd to efficiently compute more informative updates even when non-degenerate derivatives exist.
We find on two experiments that \lpgd can achieve faster convergence and better final results when compared to gradient descent.
\looseness=-1

\clearpage

\section*{Acknowledgements}
We thank the International Max Planck Research School for Intelligent Systems (IMPRS-IS) for supporting Anselm Paulus.
Georg Martius is a member of the Machine Learning Cluster of Excellence, EXC number 2064/1 – Project number 390727645.
This work was supported by the ERC - 101045454 REAL-RL.
We acknowledge the support from the German Federal Ministry of Education and Research (BMBF) through the Tübingen AI Center (FKZ: 01IS18039B).
This work was supported by the Cyber Valley Research Fund.
V\'\i t Musil was supported by the Czech Science Foundation Grant GA23-06963S.

\section*{Impact Statement}
This paper presents theory-oriented work whose goal is to advance the field of Machine Learning. There are many potential societal consequences of our work, none of which we feel must be specifically highlighted here.

\bibliographystyle{icml2024}

\appendix

\newpage
\section{Limitations}

We are aware of a few limitations of our method and are committed to transparent communication of~them.

Reformulating the \lag-proximal map as an instance of the forward solver makes the implementation simple and efficient.
However, such reformulation is not possible in all cases since it requires access to the linear coefficients of the Lagrangian via the solver interface. This same issue arises for the augmentation introduced in \Secref{sec:augmented-lagrangian}, which requires accessing the quadratic coefficients of the solver and potentially turns an LP from the forward pass into a QP on the backward pass. 
This does not apply to our implementation of \lpgd for the SCS solver \citep{odonoghue2016conic} in CVXPY, as the solver natively supports quadratic conic programs.

When considering extremely small values of $\tau>0$, the \lpgd update requires solving the optimization problem very accurately, which can be expensive. Otherwise, warm-starting the solver on the backward pass with the forward pass solution will already satisfy the stopping criterion.
However, we advocate for larger values of $\tau$ as the increased smoothing is usually beneficial. In our experiment, due to the convexity of the Lagrangian we are able to use the forward solver oracle to solve the optimization problems \eqref{eq:embedded-opt-problem} and \eqref{eq:augmented-lower-lagrangian-proximal-map} to a very high accuracy in reasonable time using the SCS solver~\citep{odonoghue2016conic}, and therefore did not observe issues arising from inexact solutions.

Finally, choosing the right combination of hyperparameters $\tau$ and $\rho$ can potentially require expensive tuning, as the different terms in the objective of the \lag-envelope \eqref{eq:augmented-lower-lppm-lpgd} can be of different magnitudes, depending on the problem at hand. A potential remedy is to normalize the terms onto a shared scale, which makes well-performing values of $\tau$ and $\rho$ more transferable and interpretable. Future work might also take inspiration from adaptive methods such as AIMLE \citep{minervini2023adaptive} for automatically tuning $\tau$ and $\rho$ on the fly.

\section{Extended Related Work}\label{sec:extended-related-work}
In this section we discuss the methods that are closest related to our framework. 
\paragraph{Direct Loss Minimization.}
In \emph{Direct Loss Minimization}~\citep{mcallester2020direct} the goal is to directly optimize a structured prediction pipeline for a given task loss such as the BLEAU score, extending previous work using structured SVMs or CRFs.
Given an input $\inp\in\sR^p$ the prediction pipeline consists of a feature map $\Psi\colon \X \times \sR^p \rightarrow \sR^k\colon (\x,\mu)\mapsto \Psi(\x, \inp)$ and a corresponding parameterized Lagrangian (called score function) $\lag(\x,\w)=\langle\w, \Psi(\x, \inp)\rangle$.
The structured prediction is then computed by solving the embedded optimization problem 
\begin{align}
    \x^*(\w, \inp)\define\argmax_{\x\in\X}\langle \w, \Psi(\x, \inp) \rangle
\end{align}
over a finite set of solutions $\X$. Finally a task-specific loss $\ell\colon \X \times \X \rightarrow \sR$ is used to compare the prediction to a label $\x_\text{true}\in\X$.
The goal is to optimize the loss over a dataset $\{(\inp_i,\x_{\text{true},i})\}_{i=1}^N$, and the authors propose to optimize it using gradient descent. They show for the case $\X = \{-1,1\}$ that the gradient can be computed using the limit of the finite difference
\begin{align*}
    \nabla_\w\ell(\x^*(\w, \inp))
    &= \pm\lim_{\tau\rightarrow 0} \inv{\tau}\bigl[\Psi(\x_{\pm\tau}, \inp) - \Psi(\x^*, \inp)\bigr]
\end{align*}
with
\begin{align}
    \x_{\pm\tau}
    \define \argmax_{\x\in\X} \langle \w, \Psi(\x, \inp) \rangle \pm \tau \ell(\x, \x_\text{true}).
\end{align}
Where the two sign cases are called the ``away-from-worse'' and the ``towards-better'' update.
The authors also discuss using relaxations of $\X$ as well has hidden variables that have similarities to our dual variables.
The method is applied to phoneme-to-speech alignment.
The DLM framework has also been generalized to non-linear objective functions \citep{song2016training, lorberbom2019direct}, always considering the limit $\tau\rightarrow 0$, with applications to action classification, object detection, and semi-supervised learning of structured variational autoencoders.
\paragraph{Blackbox Backpropagation.}
In \emph{Blackbox Backpropagation}~\citep{VlastelicaEtal2020} the authors consider embedded combinatorial optimization problems with linear cost functions. Given a (potentially high-dimensional) input $\mu\in\sR^p$ the prediction pipeline first computes the cost vector $\c\in\sR^n$ using a backbone model $\backbone\colon \sR^p \times \Theta \rightarrow \sR^n \colon (\mu,\theta)\mapsto \c$ and then solves the embedded linear optimization problem over a combinatorial space $\X\subset\sR^n$
\begin{align}
    \x^*(\c)\define\argmin_{\x\in\X} \langle \x, \c \rangle.
\end{align}
Finally the optimal solution is used as the prediction and compared to a label $\x_\text{true}\in\X$ on a given loss $\ell\colon \X \times \X \rightarrow \sR$.
Because of the discrete solution space the gradient of the loss with respect to the cost vector (and therefore also the parameters $\theta$) is uninformative, as it is either zero or undefined. 
The authors propose to replace the uninformative gradient w.r.t. $\c$ with the gradient of a piecewise-affine loss interpolation
\begin{align}
    \widetilde\ell_\tau(\c)\define\ell(\x^*(\c)) - \inv\tau \min_{\x\in\X} \langle \c, \x^*(\c) - \widetilde\x_\tau(\c)\rangle
\end{align}
where we defined
\begin{align}
    \widetilde\x_\tau(\c)\define\argmin_{\x\in\X} \langle \x, \c + \tau \nabla_\x\ell(\x^*) \rangle.
\end{align}
It has the gradient
\begin{align}
    \nabla\widetilde\ell_\tau(\c)=-\inv{\tau}[\x^* - \x^*(\c + \tau\nabla_\x\ell(\x^*))].
\end{align}
This approach has also been extended to learning discrete distributions in \cite{niepert2021implicit}.
\paragraph{Implicit Differentiation by Perturbation.}
\emph{Implicit Differentiation by Perturbation}~\citep{domke2010implicit} was proposed in the context of graphical models and marginal inference.
The parameters $\c$ now correspond to the parameters of an exponential family distribution and $\X$ is the marginal polytope. Computing the marginals requires solving the entropy-regularized linear program
\begin{align}
    \argmax_{\x\in\X} \langle \x, \c \rangle + \Omega(\x),
\end{align}
where $\Omega$ denotes the entropy. The final loss function $\ell$ depends on the computed marginals and some data $\x_\text{true}$.
Approximate marginal inference can be seen as approximating the marginal polytope with a set of linear equality constraints
\begin{align}
    \x^*(\c)\define\argmax_{\x\in\sR^n, A\x=b} \langle \x, \c \rangle + \Omega(\x),
\end{align}
which can be solved by loopy/tree-reweighted belief propagation.
The authors show that in this case, as an alternative to implicit differentiation, the gradients of the loss with respect to the parameters can be computed by
\begin{align}
    \nabla_\w\ell(\x^*(\w)) = \lim_{\tau\rightarrow 0} \inv\tau [\x^*(\c\!+\!\tau\nabla_\x\ell(\x^*)) \!-\! \x^*(\c)].
\end{align}
The method is applied to binary denoising, where the authors also test the double-sided perturbation case.
\paragraph{Identity with Projection.}
In \emph{Identity with Projection}~\citep{sahoo2022gradient} the setup is the same as in Blackbox Backpropagation. 
The goal of this method is to speed up the backward pass computation by removing the second invocation of the solver oracle on the backward pass.
In the basic version, the authors propose to replace the uninformative gradient through the solver $\nabla_\c\ell(\x^*(\c))$ by simply treating the solver as a negative identity, and returning $\Delta_\text{Id}\define -\nabla\ell(\x^*)$ instead.
Connections are drawn to the straight-through estimator.
Further, the authors identify transformations of the cost vector $P\colon \sR^k \rightarrow \sR^k$ that leave the optimal solution unchanged (\eg normalization), \ie $\x^*(P(\c))=\x^*(\c)$. 
They propose to refine the vanilla identity method by differentiating through the transformation, yielding the update $\Delta_\text{Id}\define P'_{\x^*}(-\nabla\ell(\x^*))$.
This is also shown to have an interpretation as differentiating through a projection onto a relaxation $\widetilde\X$ of the feasible space, \ie
\begin{align}
    \Delta_\text{Id}=D^*P_{\widetilde\X}(\x^*)(-\nabla\ell(\x^*)).
\end{align}
Experimentally, the method is shown to be competitive with Blackbox Backpropagation and I-MLE.
\paragraph{Smart Predict then Optimize.}
In \emph{Smart Predict then Optimize}~\citep{elmachtoub2022smart} setting the authors consider a linear program
\begin{align}
    \x^*(\c)\define\argmin_{\x\in\X} \langle \x, \c \rangle
\end{align}
as the final component of a prediction pipeline. The cost parameters $\c\in\sR^n$ are not known with certainty at test time, and are instead predicted from an input $\inp\in\sR^p$ via a prediction model $\backbone_\theta\colon \sR^p \rightarrow \sR^n \colon \mu\mapsto \c$ with parameters $\theta\in\Theta$.
What distinguishes this setup from the one considered in \eg Blackbox Backpropagation is that during training the true cost vectors $\c_\text{true}$ are available, \ie there is a dataset $\{\inp_i, \c_{\text{true},i}\}_{i=1}^N$. A naive approach would be to directly regress the prediction model onto the true cost vectors by minimizing the mean squared error
\begin{align}
    \ell_\text{MSE}(\inp, \c_\text{true}) = \inv 2 \|W_\theta(\mu) - \c_\text{true}\|.
\end{align}
However, the authors note that this ignores the actual downstream performance metric, which is the regret or SPO loss
\begin{align}
    \ell_\text{SPO}(\x^*(\c), \c_\text{true})=\langle \x^*(\c) - \x^*(\c_\text{true}), \c_\text{true} \rangle.
\end{align}
Unfortunately this loss does not have informative gradients, so the authors propose to instead optimize a convex upper bound, the SPO$+$ loss
\begin{align}
\begin{split}
    \ell_\text{SPO+}(\c, \c_\text{true}) 
    &\define \sup_{\x\in\X} \langle \x,  \c_\text{true} - 2\c\rangle 
    \\
    &\quad+ 2\langle \x^*(\c_\text{true}), \c \rangle
    - \langle \x^*(\c_\text{true}), \c_\text{true} \rangle
\end{split}
\end{align}
a generalization of the hinge loss.
Experimentally, results on shortest path problems and portfolio optimization demonstrate that optimizing the SPO$+$ loss offers significant benefits over optimizing the naive MSE loss.
\paragraph{Fenchel-Young Losses.}
\emph{Fenchel Young Losses}~\citep{blondel2020learning} were proposed in the context of structured prediction, in which the final prediction is the solution of a regularized linear program
\begin{align}
    \x^*(\c)\define\argmax_{\x\in\X} \langle \x, \c \rangle - \Omega(\x).
\end{align}
Supervision is assumed in the form of ground truth labels $\x_\text{true}$.
The authors propose to learn the parameters $\c$ by minimizing the convex Fenchel Young Loss
\begin{align*}
    \ell_\text{FY}(\c, \x_\text{true}) 
    &\define  \max_{\x\in\X} \bigl[\langle \c, \x\rangle \!-\! \Omega(\x)\bigr]  \!+\! \Omega(\x_\text{true}) \!-\! \langle \c, \x_\text{true}\rangle 
        \\
    &\phantom{:}=  \langle \c, \x_\text{true}\rangle \!+\! \Omega(\x_\text{true}) \!-\! \min_{\x\in\X} \bigl[\langle \c, \x\rangle \!+\! \Omega(\x)\bigr],
\end{align*}
for which a variety of appealing theoretical results are presented.
Many known loss functions from structured prediction and probabilistic prediction are shown to be recovered by Fenchel-Young losses for specific choices of feasible region $\X$ and regularizer $\Omega$, including the \emph{structured hinge} \citep{tsochantaridis2005large}, \emph{CRF} \citep{lafferty2001conditional}, and \emph{SparseMAP} \citep{niculae2018sparsemap} losses.

\section{Implicit Differentiation with Augmentation} \label{sec:implicit-function-theorem}

We inspect how the augmentation in \eqref{eq:augmentation} affects existing methods for computing the adjoint derivative of the augmented optimization problem
\begin{equation*} %
	\!\z_\rho^*(\w)
		= (\x_\rho^*(\w), \y_\rho^*(\w))
		\define \arg\min_{\x\in\X} \max_{\y\in\Y} \lag_\rho(\x, \y, \w).
\end{equation*}

\paragraph{Quadratic Program.}
For a symmetric positive semi-definite matrix $\H$ we can write a quadratic program with inequality constraints as
\begin{align}
\begin{split}
    (\x^*, \s^*) = &\argmin_{\x, \s\geq 0} \inv{2}\x^ T \H \x + c^Ty
        \\&
    \text{subject to} \quad A \x + b + \s = 0.
\end{split}
\end{align}
In Lagrangian form, we can write it as 
\begin{align*}
    \z^* = (\x^*, \s^*, \y^*) = &\arg\min_{\x, \s\geq 0} \max_{\y} \lag(\H, \c, A,\b, \x, \y)
\end{align*}
with the Lagrangian
\begin{align*}
    \lag(\x, \s, \y, \H, A, \b, \c) = \inv{2}\x^ T \H  \x + \c^T \x + (A\x + \b + \s)^T \y.
\end{align*}
The augmentation in \eqref{eq:augmentation} augments the Lagrangian to
\begin{align*}
    \lag_\rho&(\x, \s, \y, \H, A, \b, \c) 
        \\&
        = \inv{2}\x^ T \H  \x + \c^T \x + (A\x + \b + \s)^T \y + \tfrac{1}{2\rho}\|\x - \x^*\|^2_2\nonumber
\end{align*}
and we write
\begin{align*}
    \z_\rho^*(\H, A,\b, \c) = &\arg\min_{\x, \s\geq 0} \max_{\y} \lag_\rho(\x, \s, \y, \H, A, \b, \c).
\end{align*}
As described in \citep{amos2017optnet}, the optimization problem can be differentiated by treating it as an implicit layer via the KKT optimality conditions, which are given as
\begin{align}
    \H \x + A^T\y + \c + \tfrac{1}{\rho}(\x - \x^*) &= 0
        \\
    \diag(\y)\s  &= 0
        \\
    A\x + \b + \s &= 0
        \\
    \s&\geq 0.
\end{align}
Assuming strict complementary slackness renders the inequality redundant and the conditions reduce to the set of equations
\begin{align}\label{eq:augmented-qp-opt-conditions}
\begin{split}
    0&=F_\rho(\x, \s, \y, \H, A,\b, \c) 
        \\&
    = 
    \begin{pmatrix}
        \H \x + A^T\y + \c + \tfrac{1}{\rho}(\x - \x^*)\\
        \diag(\y)\s\\
        A\x + \b + \s
    \end{pmatrix}
\end{split}
\end{align}
which admits the use of the implicit function theorem. It states that under the regularity condition that ${\partial F_\rho}/{\partial\z}$ is invertible,
$\z_\rho^*(\w)$ can be expressed as an implicit function, and we can compute its Jacobian by linearizing the optimality conditions around the current solution
\begin{align}
    0 = \frac{\partial F_\rho}{\partial\z}\frac{\partial \z_\rho^*}{\partial \w} + \frac{\partial F_\rho}{\partial \w}
\end{align}
with
\begin{align}
    \frac{\partial F_\rho}{\partial\z} = 
    \begin{bmatrix}
        \H + \tfrac{1}{\rho} I & 0 & A^T \\
        0 & \diag(\tfrac{\y}{\s}) & I \\
        A & I & 0
    \end{bmatrix}.
\end{align}
It is now possible to compute the desired Vector-Jacobian-product as
\begin{align}
\begin{split}    
    \nabla_\w\ell(\x^*(\w)) 
        & = \frac{\partial \z^*}{\partial \w}^T\nabla\ell(\x^*)
            \\
        & = -\frac{\partial F_\rho}{\partial \w}^T\frac{\partial F_\rho}{\partial\z}^{-T}\nabla\ell(\x^*),
\end{split}
\end{align}
which involves solving a linear system. The augmentation term, therefore, serves as a regularizer for this linear system.

\paragraph{Conic Program.}
A conic program \citep{boyd2014convex} is defined as
\begin{align*}
    (\x^*, \s^*) = &\argmin_{\x,s\in\cone} c^Ty
    \quad \text{subject to} \quad A \x + b + s = 0
\end{align*}
where $\cone$ is a cone. The Lagrangian of this optimization problem is
\begin{align}
    \lag(\x, \s, \y, A, \b, \c) = \c^T \x + (A\x + \b + \s)^T \y
\end{align}
which allows an equivalent saddle point formulation given by
\begin{align*}
    \z^* = (\x^*, \s^*, \y^*) = &\arg\min_{\x,s\in\cone}\max_\y \lag(\x, \s, \y, A, \b, \c).
\end{align*}
The KKT optimality conditions are
\begin{align}
    A^T\y + \c &= 0,
        \\
    A\x + \s + \b&= 0,
        \\
    (\s, \y) &\in \cone\times \cone^*,
        \\
    \s^T \y &= 0,
\end{align}
where $\cone^*$ is the dual cone of $\cone$.
The skew-symmetric mapping
\begin{align}\label{eq:skew-symmetric-mapping}
    \Q(A, \b, \c) = \begin{bmatrix}
        0 & A^T & \c \\
        -A & 0 & \b \\
        -\c^T & -\b^T & 0
    \end{bmatrix}
\end{align}
is used in the homogenous self-dual embedding \citep{odonoghue2016conic, busseti2019solution}, a feasibility problem that embeds the conic optimization problem. 
\citet{agrawal2019differentiating} solve and differentiate the self-dual embedding.
We use the CVXPY implementation of this method as our baseline for computing the true adjoint derivatives of the optimization problem.
The augmentation in \eqref{eq:augmentation} changes the stationarity condition and, thereby, the skew-symmetric mapping as
\begin{align}\label{eq:augmented-skew-symmetric-mapping}
    \Q_\rho(A, \b, \c) = \begin{bmatrix}
        \tfrac{1}{\rho} I & A^T & \c \\
        -A & 0 & \b \\
        -\c^T & -\b^T & 0
    \end{bmatrix}.
\end{align}
We adjust the CVXPY implementation accordingly for our experiments.

\section{Relation to Mirror Descent}\label{sec:mirror-descent}

\paragraph{Standard Mirror Descent.}  %
Classical mirror descent is an algorithm for minimizing a function $\ell(\x)$ over a closed convex set $\X\subseteq \sR^n$.
The algorithm is defined by the distance-generating function (or mirror map) $\phi\colon\sR^n\rightarrow \sR$, a strictly convex continuously differentiable function.
The mirror descent algorithm also requires the assumption that the dual space of $\phi$ is all of $\sR^n$, \ie $\{\nabla \phi(\x)\mid \x \in \sR^n\} = \sR^n$, and that the gradient of $\phi$ diverges as $\|\x\|_2\rightarrow\infty$. 

The Bregman divergence of the mirror map is defined as
\begin{align}\label{eq:bregman-divergence}
    D_\phi(\x, \widehat\x) \define \phi (\x) - \phi (\widehat\x) - \langle \x-\widehat\x, \nabla \phi(\widehat\x) \rangle.
\end{align}
The lower
\footnote{The terms lower and upper are replaced with left and right in \citet{bauschke2018regularizing}.}
\emph{Bregman-Moreau envelope}~\citep{bauschke2018regularizing}
$\env^\phi_{\tau\!\f}\colon \sR^n \rightarrow \sR$ of a possibly non-smooth function $\f\colon\sR^n\to\sR$ is defined for $\tau>0$~as
\begin{align}\label{eq:lower-bregman-moreau-envelope}
    \env_{\tau\!\f}^\phi(\widehat\x) 
    \define \min_{\x} f(\x) + \tfrac{1}{\tau}D_\phi(\x, \widehat\x).
\end{align}
The corresponding lower \emph{Bregman-Moreau proximal map} $\prox^\phi_{\tau\!\f}\colon \sR^n \to \sR^n$ is given by
\begin{align}\label{eq:lower-bregman-moreau-proximal-map}
	\begin{split}
    	\prox_{\tau\!\f}^\phi(\widehat\x)
    	   \define \arg\inf\nolimits_{\x} \f(\x) + \inv{\tau} D_\phi(\x, \widehat\x)
	\end{split}
\end{align}
The superscript $\phi$ distinguishes the notation from the standard Moreau envelope \eqref{eq:moreau-envelope} and proximal map \eqref{eq:proximal-map}.

Then, the mirror descent algorithm in proximal form is given by iteratively applying the Bregman-Moreau proximal map of $\tilde\ell + I_\X$ as
\begin{align}
    \x_{k+1} 
    &= \argmin_{\x} \tilde\ell(\x) + I_\X(\x) + \tfrac{1}{\tau}D_\phi(\x, \x_k)\\
    &= \argmin_{\x\in\X} \langle \x, \nabla\ell(\x_k) \rangle + \tfrac{1}{\tau}D_\phi(\x, \x_k).
\end{align}

\paragraph{Lagrangian Mirror Descent.}
In this section, we derive \emph{Lagrangian Mirror Descent} (\lmd), an alternative algorithm to \lpgd inspired by mirror descent.
We define 
\begin{align}\label{eq:primal-objective}
    \lag_\w(\x) \define \sup_{\y\in\Y} \lag(\x, \y, \w)
\end{align}
and assume that $\lag_\w$ is strongly convex and continuously differentiable in $\x$.
As the key step, we identify the distance-generating function (mirror map) as the Lagrangian, i.e. $\phi=\lag_\w$. This has the interpretation that distances are measured in terms of the Lagrangian, similar to the intuition behind the previously defined Lagrangian divergence.

This mirror map leads to the Bregman divergence
\begin{align}
    D_{\lag_\w}(\x, \widehat\x) 
    &= \lag_\w(\x) - \lag_\w(\widehat\x) - \langle\x-\widehat\x, \nabla \lag_\w(\widehat\x)\rangle
\end{align}
satisfying
\begin{align}
    &D_{\lag_\w}(\x, \x^*) \geq 0,\\
    &D_{\lag_\w}(\x, \x^*)  = 0 \Leftrightarrow \x = \x^*(\w) \quad \text{for $\x\in\X$}.
\end{align}
We define the lower \emph{Lagrange-Bregman-Moreau envelope} at $\w$ as the Bregman-Moreau envelope at $\x^* = \x^*(\w)$, \ie
\begin{align}\label{eq:lower-lagrange-bregman-moreau-envelope}
    \ell^\phi_{\tau}(\w) 
    &\define \env_{\tau\ell+I_\X}^{\lag_\w}(\x^*) 
        \\
    &\phantom{:}= \min_{\x} \ell(\x) + I_\X(\x) + \tfrac{1}{\tau}D_{\lag_\w}(\x, \x^*)
        \\
    \begin{split}
    &\phantom{:}= \min_{\x\in\X} \ell(\x) + \tfrac{1}{\tau}(\lag(\x,\w) - \lag^*(\w) 
    \\&\qquad
    - \langle\x-\x^*, \nabla_\x \lag(\x^*, \w)\rangle),
    \end{split}
\end{align}
and the corresponding lower \emph{Lagrange-Bregman-Moreau proximal map}
\begin{align}\label{eq:lower-lagrange-bregman-moreau-proximal-map}
    \x^\phi_{\tau}(\w) 
    &\define \prox_{\tau\ell+I_\X}^{\lag_\w}(\x^*) 
        \\
    &\phantom{:}= \argmin_{\x} \ell(\x) + I_\X(\x) + \tfrac{1}{\tau}D_{\lag_\w}(\x, \x^*)
        \\
    \begin{split}
    &\phantom{:}= \argmin_{\x\in\X} \ell(\x) + \tfrac{1}{\tau}(\lag(\x,\w) - \lag^*(\w) 
        \\&\qquad
    - \langle\x-\x^*, \nabla_\x \lag(\x^*, \w)\rangle).
    \end{split}
\end{align}
Again, the superscript $\phi$ distinguishes the notation from the lower Lagrange-Moreau envelope \eqref{eq:lower-lag-moreau-envelope} and lower $\lag$-proximal map \eqref{eq:lower-lag-proximal-map}.

Similar to how we defined \lpgd as gradient descent on the Lagrange-Moreau envelope of the linearized loss, we define \emph{Lagrangian Mirror Descent} (\lmd) as gradient descent on the Lagrange-Bregman-Moreau envelope of the linearized loss, \ie
\begin{align}
\begin{split}
    \nabla_\w \widetilde\ell^\phi_{\tau}(\w) 
    &= \tfrac{1}{\tau}\nabla_\w [\lag(\widetilde\x^\phi_{\tau}, \w) - \lag(\x^*, \w) 
        \\&\qquad
    - \langle \widetilde\x^\phi_{\tau} - \x^*, \nabla_\x \lag(\x^*, \w) \rangle].
\end{split}
\end{align}
Again, the approximation allows efficiently computing the gradients using the forward solver as
\begin{align}
    \begin{split}
     \widetilde\x^\phi_{\tau}(\u, \v)
     &= \arg\min_{\x\in\X} \tau\langle \x, \nabla\ell\rangle + \langle \x, \u \rangle 
        \\&\qquad
     + \Omega(\x, \y, \v) - \langle\x, \nabla_\x \lag(\x^*, \w)\rangle
     \end{split}
        \\&
    = \x^*(\u + \tau\nabla\ell - \nabla_\x \lag(\x^*, \w), \v).
\end{align}

If $\X=\sR^n$, then the original optimization problem \eqref{eq:embedded-opt-problem} is unconstrained and we have the optimality condition $\nabla_\x\lag(\x^*,\w)=0$.
Therefore, the Bregman divergence
\begin{align*}
    D_{\lag_\w}(\x, \x^*) 
    &= \lag_\w(\x) - \lag_\w(\x^*) - \langle\x-\x^*, \nabla \lag_\w(\x^*)\rangle\nonumber
        \\&
    = \lag(\x,\w) - \lag(\x^*,\w)
        \\&
    = \lag(\x,\w) - \lag^*(\w)
    = D_\lag(\x|\w)
\end{align*}
coincides with the Lagrangian divergence \eqref{eq:lagrangian-divergence}.
It follows that, in this case, the Lagrange-Moreau envelope and \lpgd coincide with the Lagrange-Bregman-Moreau envelope and \lmd, respectively.

\section{Extension to General Loss Functions}\label{sec:general-loss}
\paragraph{Loss on Dual Variables.}
In the main text we considered losses depending only on the primal variables, \ie $\ell(\x)$.
For a loss on the dual variables $\ell(\y)$, if we assume strong duality of \eqref{eq:optimal-lagrangian}, we can reduce the situation to the primal case, as
\begin{align}
    \y^*(\w) &=  \argmax_{\y\in\Y} \min_{\x\in\X}\lag(\x, \y, \w) \\
    &= \argmin_{\y\in\Y} \max_{\x\in\X} -\lag(\x, \y, \w).
\end{align}
This amounts to simply negating the Lagrangian in all equations while swapping $\x$ and $\y$.

\paragraph{Loss on Primal and Dual Variables.}
The situation becomes more involved for a loss function $\L(\x,\y)$ depending on both primal and dual variables. 
If it decomposes into a primal and dual component, \ie $\L(\x, \y)=\ell_p(\x)+\ell_d(\y)$, we can compute the envelopes of the individual losses independently. Note that a linearization of the loss $\tilde \L$ trivially decomposes this way. Adding the envelopes together yields a combined lower and upper envelope for the total loss as
\begin{align}
    \begin{split}
    (\ell_p)_{\tau}(\w) + (\ell_d)_{\tau}(\w)
    = &\min_{\x\in\X}\max_{\y\in\Y} [\tfrac{1}{\tau}\lag(\x, \y, \w) + \ell_p(\x)] 
        \\
    - &\max_{\y\in\Y} \min_{\x\in\X} [\tfrac{1}{\tau}\lag(\x, \y, \w) - \ell_d(\y)],
    \end{split}
        \\
    \begin{split}
    (\ell_d)^{\tau}(\w) + (\ell_p)^{\tau}(\w)
    = &\max_{\y\in\Y} \min_{\x\in\X} [\tfrac{1}{\tau}\lag(\x, \y, \w) + \ell_d(\y)]  
        \\
    - &\min_{\x\in\X}\max_{\y\in\Y} [\tfrac{1}{\tau}\lag(\x, \y, \w) - \ell_p(\x)].
    \end{split}
\end{align}
Assuming strong duality of these optimization problems, this leads to
\begin{align*}
    \begin{split}
    (\ell_p)_{\tau}(\w) + (\ell_d)_{\tau}(\w)
    = &\min_{\x\in\X}\max_{\y\in\Y} [\tfrac{1}{\tau}\lag(\x, \y, \w) + \ell_p(\x)]
        \\
    -  &\min_{\x\in\X}\max_{\y\in\Y} [\tfrac{1}{\tau}\lag(\x, \y, \w) - \ell_d(\y)],
    \end{split}
        \\
    \begin{split}
    (\ell_d)^{\tau}(\w) + (\ell_p)^{\tau}(\w)
    = &\min_{\x\in\X}\max_{\y\in\Y} [\tfrac{1}{\tau}\lag(\x, \y, \w) + \ell_d(\y)]  
        \\
    -  &\min_{\x\in\X}\max_{\y\in\Y} [\tfrac{1}{\tau}\lag(\x, \y, \w) - \ell_p(\x)].
    \end{split}
\end{align*}
Strong duality holds in particular for a linearized loss $\tilde \L$, as this only amounts to a linear perturbation of the original optimization problem.
Unfortunately, computing the average envelope 
\begin{align*}
    \!\!(\ell_p)\mathmiddlescript{\tau}(\w) + (\ell_d)\mathmiddlescript{\tau}(\w)
    = \tfrac{1}{2}\biggl\{\!
    &\min_{\x\in\X}\max_{\y\in\Y} [\tfrac{1}{\tau}\lag(\x, \y, \w) + \ell_p(\x)] \nonumber
        \\
    -  \!&\min_{\x\in\X}\max_{\y\in\Y} [\tfrac{1}{\tau}\lag(\x, \y, \w) - \ell_d(\y)]\nonumber
        \\
    + &\min_{\x\in\X}\max_{\y\in\Y} [\tfrac{1}{\tau}\lag(\x, \y, \w) + \ell_d(\y)]  \nonumber
        \\
    -  \!&\min_{\x\in\X}\max_{\y\in\Y} [\tfrac{1}{\tau}\lag(\x, \y, \w) - \ell_p(\x)] \biggr\}
\end{align*}
or its gradient would now require four evaluations of the solver, which can be expensive. To reduce the number of evaluations, we instead ``combine'' the perturbations of the primal and dual loss, \ie we define
\begin{align}
    \L_{\tau}(\w)
    &\define \min_{\x\in\X} \max_{\y\in\Y} \L(\x,\y) + \tfrac{1}{\tau}[\lag(\x, \y, \w) - \lag^*(\w)] \label{eq:general-lower-lag-moreau-envelope}\\
    \L^{\tau}(\w)
    &\define \min_{\x\in\X} \max_{\y\in\Y} \L(\x, \y) + \tfrac{1}{\tau}[\lag(\x, \y, \w) - \lag^*(\w)] \\
    \begin{split}
    \L\mathmiddlescript{\tau}(\w) 
    &\define \tfrac{1}{2} \{ \L_{\tau}(\w) + \L^{\tau}(\w) \}\\
    &\phantom{:}= \min_{\x\in\X} \max_{\y\in\Y}  [\tfrac{1}{\tau}\lag(\x, \y, \w) + \L(\x, \y)] 
        \\&\qquad
    - \min_{\x\in\X} \max_{\y\in\Y} [\tfrac{1}{\tau}\lag(\x, \y, \w) - \L(\x, \y)].
    \end{split}
\end{align}
The above definitions are valid even when we do not have strong duality of~\eqref{eq:optimal-lagrangian} and~\eqref{eq:general-lower-lag-moreau-envelope}. Note that $\L_{\tau}$ and $\L^{\tau}$ are not necessarily lower and upper bounds of the loss anymore.
However, these combined envelopes also apply to loss functions that do not separate into primal and dual variables, and computing their gradients requires fewer additional solver evaluations.
For $\L(\x, \y) = \ell(x)$ we have
\begin{align}
    \L_{\tau}(\w) &= \ell_{\tau}(\w), &
    \L\mathmiddlescript{\tau}(\w) &= \ell\mathmiddlescript{\tau}(\w), &
    \L^{\tau}(\w) &= \ell^{\tau}(\w). 
\end{align}
In some of the proofs we will work with $\L$ instead of $\ell$ for full generality and reduce the situation to a primal loss $\ell$ with the relations above.

\section{Proofs}\label{sec:proofs}
\begin{lemma}[Equation~\eqref{eq:lagrangian-divergence-property}]\label{lemma:1}
    It holds that
    \begin{align}
	\def\myspace{\hspace{0.5em plus 0.5em}}
  \!\!D^*_\lag(\x | \w) = 0
		\myspace\text{if and only if}\myspace
	\text{$\x$ minimizes \eqref{eq:embedded-opt-problem}}
	\myspace\text{for $\x\in\X$}.\!\!\nonumber
    \end{align}
\end{lemma}
\begin{proof}[Proof of \Lemmaref{lemma:1}]
    If $\x\in\X$ minimizes~\eqref{eq:embedded-opt-problem}, this means
    \begin{align}
        \sup_{\y\in\Y}\lag (\x, \y, \w)
        &= \inf_{\widetilde\x\in\X}\sup_{\y\in\Y}\lag (\widetilde\x, \y, \w)
    \end{align}
    and therefore
    \begin{align}
        D^*_\lag(\x | \w) 
        &= \sup_{\y\in\Y}\bigl[\lag (\x, \y, \w) - \lag^*(\w)\bigr]
            \\
        &= \min_{\widetilde\x\in\X}\max_{\y\in\Y}\bigl[\lag (\widetilde\x, \y, \w)\bigr] - \lag^*(\w)
        = 0.
    \end{align}
    If $D^*_\lag(\x | \w)=0$, we have from the definition of the Lagrangian divergence~\eqref{eq:lagrangian-divergence} that
    \begin{align}
        \sup_{\y\in\Y}\lag (\x, \y, \w) 
        &= \lag^*(\w)
            \\
        &= \min_{\widetilde\x\in\X}\max_{\y\in\Y}\lag (\widetilde\x, \y, \w)
    \end{align}
    and hence $\x$ is a minimizer of~$\eqref{eq:embedded-opt-problem}$.
\end{proof}

\begin{lemma}[Equation~\eqref{eq:lppm}]\label{lemma:2}
    Assume that $\lag,\ell\in\mathcal{C}^1$ and assume that the solution mappings of optimization~\threeeqrefs{eq:embedded-opt-problem}{eq:lower-lag-moreau-envelope}{eq:upper-lag-moreau-envelope} admit continuous selections $\x^*(\w), \x_\tau(\w), \x^\tau(\w)$ at $\w$. Then
    \begin{align}
        	\nabla \ell_{\tau}(\w) 
            &= \,\inv{\tau}\nabla_\w\bigl[\lag(\w, \z_{\tau}) - \lag(\z^*, \w)\bigr],
            \\
        	\nabla \ell^{\tau}(\w) 
            &= \inv{\tau}\nabla_\w\bigl[\lag(\z^*, \w) - \lag(\z^{\tau}, \w)\bigr].
    \end{align}
\end{lemma}    

\begin{proof}[Proof of~\Lemmaref{lemma:2}]
    We assumed that $\z^*(\w)$ is a selection of the solution set continuous at $\w$. 
    We also assumed that $\lag$ and $\ell$ are continuously differentiable.
    It then follows that
    \begin{align}
        \nabla_\w \ell_{\tau}(\w) 
            &= \nabla_\w\min_{\x\in\X} \max_{\y\in\Y} \ell(\x) + \inv{\tau} D_\lag(\x, \y| \w)
                \\
            \begin{split}
            &= \nabla_\w\bigl[\min_{\x\in\X}\max_{\y\in\Y} [\ell(\x) + \inv{\tau} \lag(\x, \y, \w)] 
                \\
            &\qquad - \min_{\x\in\X}\max_{\y\in\Y} \inv{\tau} \lag(\x, \y, \w) \bigr]
            \end{split}
                \\
            \begin{split}
            &= \nabla_\w\bigl[\min_{\x\in\X}\max_{\y\in\Y} \ell(\x) + \inv{\tau} \lag(\x, \y, \w)\bigr] 
                \\
            &\qquad - \nabla_\w\bigl[\min_{\x\in\X}\max_{\y\in\Y} \inv{\tau} \lag(\x, \y, \w)\bigr]
            \end{split}
                \\
            \begin{split}
            &= \nabla_\w\bigl[\ell(\x_\tau) + \inv{\tau} \lag(\x_\tau, \y_\tau, \w)\bigr] 
                \\
            &\qquad - \nabla_\w\bigl[\inv{\tau} \lag(\x^*, \y^*, \w)\bigr]
            \end{split}
                \\
            &= \inv{\tau}\nabla_\w\bigl[\lag(\w, \z_{\tau}) - \lag(\z^*, \w)\bigr]
    \end{align}
    In the fourth equation we used the result by~\citet[Proposition~4.1]{oyama2018on}.
    The proof for the upper envelope is analogous.
\end{proof}

\begingroup
\def\thetheorem{\ref{prop:lipschitz}}
\begin{proposition}
    Assume that $\lag$ is $L$-Lipschitz continuous in $\w$. 
    Then $\ell_\tau,\ell\mathmiddlescript\tau,\ell^\tau$ are $\tfrac{2L}{\tau}$-Lipschitz continuous in $\w$.
\end{proposition}
\addtocounter{theorem}{-1}
\endgroup

\begin{proof}[Proof of Proposition~\ref{prop:lipschitz}]
    We have the optimal Lagrangian
    \begin{align}
        \lag^*(\w) = \min_{\x\in\X}\max_{\y\in\Y} \lag(\x,\y,\w)
    \end{align}
    and define
    \begin{align}
        \lag_\tau^*(\w) \define \min_{\x\in\X}\max_{\y\in\Y} \lag(\x,\y,\w) + \tau\ell(\x).
    \end{align}
    Then the lower $\lag$-envelope may be written as
    \begin{align}\label{eq:enveloe-as-diff}
        \ell_\tau(\w) = \inv\tau\bigl[ \lag_\tau^*(\w) - \lag^*(\w)\bigr].
    \end{align}
    If $\lag$ is $L$-Lipschitz in $\w$ for every $\x\in\X,\y\in\Y$, where $L$ is independent of $\x$ and $\y$, then both $\lag^*$ and $\lag^*_\tau$ are $L$-Lipschitz, see \citet[Proposition~1.32]{lipschitz2018weaver}. Hence, $\ell_\tau$ is $\tfrac{2L}{\tau}$-Lipschitz.
    The cases of $\ell\mathmiddlescript\tau$ and $\ell^\tau$ are analogous.
\end{proof}

\begingroup
\def\thetheorem{\ref{prop:tau-to-zero}}
\begin{proposition}
    Assume $\lag,\ell$ are lower semi-continuous and $\ell$ is finite-valued on $\X$.
    Let $\w$ be a parameter for which
    \begin{align}
        X^*(\w)
        &\define \{\x\in\X \mid \D^*_\lag(\x |\w) = 0\}
    \end{align}
    is nonempty. 
    Then it holds that
    \begin{align}\label{eq:tau-to-zero-envelope-app}
        \lim_{\tau\rightarrow 0}\ell_\tau(\w)
        =\min_{\x^*\in{X^*(\w)}} \ell(\x^*)
    \end{align}
    and
    \begin{align}
        \lim_{\tau\rightarrow 0}\x^*_\tau(\w)
        \in \argmin_{\x\in X^*(\w)} \ell(\x^*)
    \end{align}
    whenever the limit exists.
\end{proposition}
\addtocounter{theorem}{-1}
\endgroup

\begin{proof}[Proof of Proposition~\ref{prop:tau-to-zero}]
    Throughout the proof, let $\w$ fixed, and hence we omit the dependence in the notation. For $\tau>0$ we define $f_\tau\colon\X \rightarrow \sR$ by
    \begin{align}\label{eq:f-definition}
        f_\tau(\x) = \ell(\x) + \inv\tau D_\lag^*(\x)
        \quad \text{for $\x\in\X$.}
    \end{align}
    We have the monotonicity
    \begin{align}
        f_\tau \leq f_{\tau'} \quad \text{whenever} \quad \tau' \leq \tau,
    \end{align}
    since $D^*_\lag(\x)\geq 0$, and taking the infimum over $\x\in\X$ on both sides leads to
    \begin{align}
        \ell_\tau \leq \ell_{\tau'} \quad \text{whenever} \quad \tau' \leq \tau.
    \end{align}
    Moreover, for any $\tau>0$ and all $\x^*\in X^*$, it is
    \begin{align}
        \ell_\tau
            &
        = \min_{\x\in\X}f_\tau(\x) 
        \leq f_\tau(\x^*) 
        = \ell(\x^*) \label{eq:f-upper-bound}.
    \end{align}

    Therefore $\ell_\tau$ is bounded and monotone, hence the limit as $\tau\rightarrow 0^+$ exists and
    \begin{align}\label{eq:upper-bound-l0}
        \ell_0
            &
        \define \lim_{\tau\rightarrow 0^+} \ell_\tau
        \leq \inf_{\x^*\in X^*} \ell(\x^*).
    \end{align}
    
   Denote by $(\x_\tau)_{\tau>0}$ minimizers of $\ell_\tau$, \ie $\ell_\tau = f_\tau(\x_\tau)$. Assume that $\x_\tau \rightarrow \x_0\in\X$ as $\tau\rightarrow 0^+$. 
    We show that $\x_0$ minimizes $\ell(\x)$ over $X^*$.

    Since $\ell$ is lower semi-continuous, we have that
    \begin{align}
        \liminf_{\tau\rightarrow 0^+} \ell(\x_\tau) \geq \ell(\x_0)
    \end{align}
    and, by definition, we also have that
    \begin{align}\label{E:ell_tau_in_proof}
        \ell_\tau = \ell(\x_\tau) + \inv\tau D_\lag^*(\x_\tau).
    \end{align}
    Taking the $\liminf$ as $\tau\rightarrow 0^+$ in \eqref{E:ell_tau_in_proof} yields
    \begin{align}
    \begin{split}        
        \ell_0
            &
        = \liminf_{\tau\rightarrow 0^+} \ell_\tau
            \\&
        = \liminf_{\tau\rightarrow 0^+} \bigl[\ell(\x_\tau) + \inv\tau D_\lag^*(\x_\tau)\bigr]
            \\&
        \geq  \liminf_{\tau\rightarrow 0^+} \ell(\x_\tau) + \liminf_{\tau\rightarrow 0^+}\inv\tau D_\lag^*(\x_\tau)
            \\&
        \geq  \ell(\x_0) + \liminf_{\tau\rightarrow 0^+}\inv\tau D_\lag^*(\x_\tau).\label{eq:divergence-limit-bounded}
    \end{split}
    \end{align}
    Next, since $D_\lag^*(\x)\geq 0$, inequality~\eqref{eq:divergence-limit-bounded}
    reduces to $\ell_0\ge\ell(\x_0)$, which together with
    estimate~\eqref{eq:upper-bound-l0} gives
    \begin{align}
        \ell(\x_0) \leq \inf_{\x^*\in X^*} \ell(\x^*).
    \end{align}
    It remains to show that $\x_0 \in X^*$. From \eqref{eq:divergence-limit-bounded} we know
    \begin{align}\label{eq:liminf-divergence-bound}
        \liminf_{\tau\rightarrow 0^+}\inv\tau D_\lag^*(\x_\tau) < \infty.
    \end{align}
    We assume lower semi-continuity of $\lag(\x,\y)$ in $\x$ for every $\y\in\Y$, therefore the Lagrangian divergence
    \begin{align}
        D_\lag^*(\x) = \sup_{\y\in\Y} \lag(\x,\y) - \lag^*
    \end{align}
    is also lower semi-continuous, hence
    \begin{align}
        \liminf_{\tau\rightarrow 0^+} D_\lag^*(\x_\tau) \geq D^*_\lag(\x_0).
    \end{align}
    Together with \eqref{eq:liminf-divergence-bound}, this gives that $D_\lag^*(\x_0) = 0$, or equivalently, $\x_0 \in X^*$.
\end{proof}

\begingroup
\def\thetheorem{\ref{thm:limit-gradient}}
\begin{theorem}
    Assume that $\lag\in\mathcal{C}^2$ and 
    assume that the solution mapping of optimization~\eqref{eq:embedded-opt-problem} admits a differentiable selection $\x^*(\w)$ at $\w$. Then
    \begin{align}
        \lim_{\tau\rightarrow 0} \nabla  \widetilde\ell_{\tau}(\w) 
        = \nabla_{\w} \ell(\x^*(\w))
        = \lim_{\tau\rightarrow 0} \nabla  \widetilde\ell^{\tau}(\w).
    \end{align}
\end{theorem}
\addtocounter{theorem}{-1}
\endgroup

\begin{proof}[Proof of~\Thmref{thm:limit-gradient}]
    In this proof we work in the more general setup described in \appref{sec:general-loss}, in which the loss $\L$ can depend on both primal and dual optimal solutions.
    We aim to show that for a linear loss approximation $\widetilde\L$, the \lpgd update recovers the true gradient as $\tau$ approaches zero.
    We again assume the same form of the Lagrangian as in~\eqref{eq:lag-decomposition}
    \begin{align}
        \lag(\z, \w) = \langle\z, \u\rangle + \Omega(\z, \v)
    \end{align}
    with $\w=(\u,\v)$ and get from~\citet[Proposition~4.1]{oyama2018on}
    \begin{align}
        \nabla_{\u} \lag^*(\u, \v) &= \nabla_{\u} \lag(\z^*, \u, \v) = \z^*(\u, \v).
    \end{align}
    We define
    \begin{align}
        \dbw\w \define \begin{pmatrix}\nabla_\z \L \\ 0 \end{pmatrix}.
    \end{align}
    Then it holds that
    \begin{align}
        \lim_{\tau\rightarrow 0} &\nabla_\w  \widetilde\L_{\tau}(\w)\nonumber 
            \\
        &= \lim_{\tau\rightarrow 0} \tfrac{1}{\tau} [\nabla_\w\lag^*(\w + \tau \dbw\w) - \nabla_\w\lag^*(\w)]
            \\
        &= \frac{\partial^2 \lag^*}{\partial^2 \w} \dbw\w
        = \frac{\partial^2 \lag^*}{\partial^2 \w}^T \dbw\w
            \\
        &= \frac{\partial^2 \lag^*}{\partial^2 \w}^T  \begin{pmatrix}\nabla_\z \L \\ 0 \end{pmatrix}
            \\
        &= \frac{\partial^2 \lag^*}{\partial \w \partial \u}^T \nabla_\z \L
        = \frac{\partial (\nabla_\u  \lag^*)}{\partial \w}^T \nabla_\z \L
            \\
        &= \frac{\partial\z^*}{\partial \w}^T \nabla_{\z} \L
        = \nabla_\w \L(\z^*(\w)).
    \end{align}
    The main step in this derivation appears in the second-to-last equality by identifying the Jacobian of the solution mapping as a sub-matrix of the Hessian of the optimal Lagrangian function, which is a symmetric matrix under the conditions of Schwarz's theorem. A sufficient condition for this is the assumption that $\lag\in\mathcal{C}^2$.%
    \footnote{Note that a similar derivation already appeared in \citep{domke2010implicit}, but only for primal variables with linear parameters and without considering the benefits of finite values of $\tau$.}
    Exploiting the symmetry of the Hessian then allows computing the gradient, which is a co-derivative (backward-mode, vector-jacobian-product), as the limit of a finite-difference between two solver outputs, which usually only computes a derivative (forward-mode, jacobian-vector-product) from input perturbations $\dfw\w$ as
    \begin{align}
        \dfw \z
        &= \frac{\partial \z^*}{\partial \w} \dfw \w  
            \\
        &= \lim_{\tau\rightarrow 0}\tfrac{1}{\tau}[\z^*(\w+\tau \dfw \w) - \z^*(\w)]
            \\
        &= \lim_{\tau\rightarrow 0}\tfrac{1}{\tau}[\nabla_\u \lag^*(\w+\tau \dfw \w) - \nabla_\u \lag^*(\w)].
    \end{align}
    We observe that the finite-difference appearing in the \lpgd update is the co-derivative counterpart of a finite-difference approximation of the derivative.
    This observation also fosters an interpretation of finite $\tau$ in the \lpgd update:
    In forward-mode, checking how the solver reacts to finite perturbation of the parameters intuitively provides higher-order information than linear sensitivities to infinitesimal perturbations via derivatives.
    In backward-mode, the finite-difference in \lpgd update has the equivalent advantage over standard co-derivatives, by capturing higher-order information instead of linear sensitivities.
    
    Note that for $\L(\x,\y) = \ell(\x)$ this reduces to
    \begin{align}
        \lim_{\tau\rightarrow 0} \nabla_\w  \tilde\ell_{\tau}(\w) 
        = \nabla_\w \ell(\x^*(\w)).
    \end{align}
    An analogous proof and discussion also hold for the upper envelope $\tilde\ell^{\tau}$ and average envelope $\tilde\ell\mathmiddlescript{\tau}$, corresponding to the co-derivative counterparts of the left-sided and central finite-difference approximations of the derivative, respectively.
\end{proof}

\begingroup
\def\thetheorem{\ref{prop:tau-to-infty}}
\begin{proposition}
    Assume $\lag,\ell$ are lower semi-continuous and $\ell$ is finite-valued on $\X$.
    Let $\w$ be a parameter for which
    \begin{align}
        \widehat\X(\w)
            &\define \{\x\in\X \mid \D^*_\lag(\x |\w) < \infty\}
    \end{align}
    is nonempty. Then
    \begin{align}
        \lim_{\tau\rightarrow\infty}\x_\tau(\w)
        \in \argmin_{\x\in{\widehat\X(\w)}} \ell(\x)
    \end{align}
    whenever the limit exists. For a linearized loss, we have
    \begin{align}
        \lim_{\tau\rightarrow\infty}\widetilde\x_\tau(\w)
        \in \argmin_{\x\in{\widehat\X(\w)}} \langle \x, \nabla\ell\rangle
        = \x_{FW}(\w),
    \end{align}
    where $\x_{FW}$ is the solution to a Frank-Wolfe iteration LP~\citep{frank1956algorithm}
\end{proposition}
\addtocounter{theorem}{-1}
\endgroup

\begin{proof}[Proof of Proposition~\ref{prop:tau-to-infty}]
    Throughout the proof, let $\w$ be a fixed parameter, we therefore omit the dependence in the notation. For $\tau>0$ we use $f_\tau\colon\X \rightarrow \sR$ defined in \eqref{eq:f-definition}.
    Since $D_\lag^*\geq 0$ on $\X$, we have that
    \begin{align}\label{eq:f-nonincreasing}
        f_\tau\geq f_{\tau'} \geq \ell \quad \text{on $\X$}
    \end{align}
    whenever $\tau'\geq \tau$.
    Now let $(\x_\tau)_{\tau>0}$ be minimizers of~\eqref{eq:f-definition} such that $\x_\tau \rightarrow \x_\infty\in\X$ as $\tau\rightarrow\infty$. 
    We show that $\x_\infty$ minimizes $\ell(\x)$ over $\widehat\X$.

    To this end, let $(\tau_n)_{n=1}^\infty$ be a non-decreasing sequence such that $\tau_n\rightarrow\infty$ and denote $\x_n=\x_{\tau_n}$ and $f_n=f_{\tau_n}$, for short.
    Since $\widehat\X$ is nonempty and $\ell$ is finite-valued on $\X$, we have that $f_n(\x_n)<\infty$ and $\x_n\in\widehat\X$ for all $n\in\mathbb{N}$.
    
    We assume lower semi-continuity of $\lag(\x,\y)$ in $\x$ for every $\y\in\Y$, therefore the Lagrangian divergence
    \begin{align}
        D_\lag^*(\x) = \sup_{\y\in\Y} \lag(\x,\y) - \lag^*
    \end{align}
    is also lower semi-continuous, \ie
    \begin{align}
        \liminf_{n\rightarrow \infty} D_\lag^*(\x_n) \geq D^*_\lag(\x_\infty).
    \end{align}
    As $\x_n\in\widehat\X$ for all $n\in\mathbb{N}$ it follows that 
    \begin{align}
        \infty > \liminf_{n\rightarrow \infty} D_\lag^*(\x_n) \geq D^*_\lag(\x_\infty),
    \end{align}
    and therefore $\x_\infty\in\widehat\X$ as well.
    
    Next, by~\eqref{eq:f-nonincreasing}, the sequence $f_n(\x_n)$ for $n\in\mathbb{N}$ is nonincreasing and bounded from below (by a minimum of $\ell$ on $\overline{\{\x_n\mid n\in\mathbb{N}\}}$). 
    Therefore, $f_n(\x_n)$ is convergent.
    By lower semi-continuity of $\ell$, we have that $\liminf_{n\rightarrow\infty}\ell(\x_n)\geq\ell(\x_\infty)$ and thus $\inv{\tau_n}D_\lag^*(\x_n|\w)$ converges to some $c\geq 0$.
    Altogether, for any $\hat\x\in\widehat\X(\w)$, we have
    \begin{align}
        \ell(\x_\infty) + c
        &\leq \liminf_{n\rightarrow\infty} \ell(\x_n) + \liminf_{n\rightarrow\infty} \inv{\tau_n}D_\lag^*(\x_n|\w)
            \nonumber\\
        &\leq \liminf_{n\rightarrow\infty}\bigl[\ell(\x_n) + \inv{\tau_n}D_\lag^*(\x_n|\w)\bigr]
            \nonumber\\
        &\leq \lim_{n\rightarrow\infty} \ell(\hat\x) + \inv{\tau_n}D_\lag^*(\hat\x|\w)
            \nonumber\\
        &= \ell(\hat\x).
    \end{align}
    The particular choice $\hat\x=\x_\infty$ shows that $c=0$.
    Therefore $\ell(\x_\infty)\leq \ell(\hat\x)$ for any $\hat\x\in\widehat\X$ as desired.
    The second part of the proposition follows directly by taking a loss linearization.
\end{proof}

\begingroup
\def\thetheorem{\ref{prop:tau-to-infty-augmented}}
\begin{proposition}
    Assume $\ell$ is continuous and finite-valued on $\X$. Let $\w$ be a parameter for which $\widehat\X(\w)$ is nonempty.
    The primal lower $\lag$-proximal map \eqref{eq:augmented-lower-lagrangian-proximal-map} turns into the standard proximal map \eqref{eq:proximal-map}
    \begin{align}
        \lim_{\tau\rightarrow\infty}\x_{\tau\rho}(\w)
        &= \argmin_{\x\in{\widehat\X(\w)}} \bigl[\ell(\x) + \inv{2\rho} \|\x - \x^*\|_2^2\bigr]
            \\
        &= \prox_{\rho\ell + I_{\widehat\X(\w)}}(\x^*),
    \end{align}
    whenever the limit exists. For a linearized loss, it reduces to the Euclidean projection onto $\widehat\X(\w)$
    \begin{align}
        \lim_{\tau\rightarrow\infty}\widetilde\x_{\tau\rho}(\w)
        &= \argmin_{\x\in{\widehat\X(\w)}} \bigl[\langle \x, \nabla \ell\rangle + \inv{2\rho}\|\x - \x^*\|_2^2\bigr]
            \\
        &= P_{\widehat\X(\w)}(\x^* - \rho \nabla\ell).
    \end{align}
\end{proposition}
\addtocounter{theorem}{-1}
\endgroup

\begin{proof}[Proof of Proposition~\ref{prop:tau-to-infty-augmented}]
The proof is analogous to that of Proposition~\ref{prop:tau-to-infty} with
$f_\tau\colon \X \rightarrow \sR$, where
\begin{equation*}
     f_\tau(x)=\ell(\x) + \inv\tau \D^*_\lag(\x |\w) + \inv{2\rho}\| \x - \x^*\|_2^2.
     \qedhere
\end{equation*}
\end{proof}

\begingroup
\def\thetheorem{\ref{prop:approximation-error}}
\begin{proposition}
    Assume that the solvers for \twoeqrefs{eq:embedded-opt-problem}{eq:lower-lag-proximal-map} return $\varepsilon$-accurate solutions $\widehat\z^*,\widehat\z_\tau$ with $\delta$-accurate objective values, \ie
    \begin{align}
        \|\widehat\z^* - \z^*\| &\leq \varepsilon,
        \\
        |\lag(\widehat\z^*,\w) - \lag(\z^*,\w)| &\leq \delta,
        \\
        \|\widehat\z_\tau - \z_\tau\| &\leq \varepsilon,
        \\
        \bigl|[\lag(\widehat\z_\tau,\w) + \tau\ell(\widehat\x_\tau)] - [\lag(\z_\tau,\w) + \tau\ell(\x_\tau)]\bigr| &\leq \delta.
    \end{align}
    Then for the approximate lower $\lag$-envelope
    \begin{align}
        \widehat\ell_\tau(\w) \define \ell(\widehat\x_\tau) + \inv\tau \bigl[\lag(\widehat\z_\tau,\w) - \lag(\widehat\z^*,\w)\bigr]
    \end{align}
    it holds that
    \begin{align}
        |\widehat\ell_\tau(\w) - \ell_\tau(\w)| \leq \frac{2\delta}{\tau}.
    \end{align}
    Moreover, if $\nabla_\w\lag(\x,\y,\w)$ is $L$-Lipschitz continuous in $(\x,\y)\in\X\times\Y$, it holds that
    \begin{align}
        \|\nabla\widehat\ell_\tau(\w) - \nabla\ell_\tau(\w)\| \leq \frac{2L\varepsilon}{\tau}.
    \end{align}
\end{proposition}
\addtocounter{theorem}{-1}
\endgroup

\begin{proof}[Proof of Proposition~\ref{prop:approximation-error}]
    We have
    \begin{align*}
        &|\widehat\ell_\tau(\w) - \ell_\tau(\w)|
        \\
        &\quad\leq \bigl|\bigl[\ell(\widehat\x_\tau) + \inv\tau \lag(\widehat\z_\tau,\w)\bigr] - 
        \bigl[\ell(\x_\tau) + \inv\tau \lag(\z_\tau,\w)\bigr]\bigr|
        \\
        &\qquad+ \bigl|\inv\tau\lag(\widehat\z^*,\w) - \inv\tau\lag(\z^*,\w)\bigr|
        \leq \tfrac{2\delta}{\tau},
    \end{align*}
    which shows the first part of the proposition.
    The second part follows from
    \begin{align}
    \begin{split}
        &\bigl\|\nabla\widehat\ell_\tau(\w) - \nabla\ell_\tau(\w)\bigr\|
        \\
        &\quad\leq \bigl\|\inv\tau\bigl[\nabla_\w\lag(\widehat\z_\tau,\w) - \nabla_\w\lag(\widehat\z^*,\w)\bigr] 
        \\
        &\qquad+ \inv\tau\bigl[\nabla_\w\lag(\z_\tau,\w) - \nabla_\w\lag(\z^*,\w)\bigr]\bigr\|
        \\
        &\quad\leq \inv\tau\bigl\|\nabla_\w\lag(\widehat\z_\tau,\w) - \nabla_\w\lag(\z_\tau,\w)\bigr\|
        \\
        &\qquad+ \inv\tau\bigl\|\nabla_\w\lag(\widehat\z^*,\w) - \nabla_\w\lag(\z^*,\w)\bigr\|
        \\
        &\quad\leq \tfrac{L}{\tau}\|\widehat\z_\tau - \z_\tau\| + \tfrac{L}{\tau}\|\widehat\z^* - \z^*\|
        \leq\tfrac{2L\varepsilon}{\tau}.
    \end{split}
    \end{align}
    The approximation of the upper $\lag$-envelope has an analogous proof.
\end{proof}

\section{Experiments}\label{sec:experimental-details}

\subsection{Learning the Rules of Sudoku}\label{sec:experimental-details-sudoku}

Instead of the mini-Sudoku case ($4\times 4$ grid) in \citet{amos2017optnet}, we consider the full $9\times 9$ Sudoku grid. 
The Sudoku board is modelled as a one-hot-encoding, with incomplete input and solved label $\x_\text{inc}, \x_\text{true}\in\{0,1\}^{9\times 9 \times 9}$.
The optimization problem is modelled as a generic box-constrained linear program
\begin{align}\label{eq:sudoku-linear-program}
\begin{split}
    \x^*(A, \b; \x_\text{inc}) 
    = &\argmin_{\x\in\X} \langle \x, \x_\text{inc} \rangle 
        \\
    &\text{subject to}\quad  A\x + \b = 0
\end{split}
\end{align}
with $\X=[0, 1]^{9\times 9 \times 9}$ and a sufficient number of constraints $m$ to represent the rules of the LP Sudoku.\!%
\footnote{The formulation as an LP differs from the original formulation by \citet{amos2017optnet}, in which a quadratic regularizer has to be added to meet the method requirements.}
The saddle-point formulation of \eqref{eq:sudoku-linear-program} is
\begin{align}\label{eq:sudoku-saddlepoint}
    \z^*(A, \b; \x_\text{inc}) 
    = \arg\min_{\x\in\X}\max_{\y\in\sR^m} \lag(\x, \y, A, \b, \x_\text{inc}),
\end{align}
with $\lag(\x, \y, A, \b, \c) = \langle \x, \c \rangle + \langle \y, A\x + \b \rangle$.
Note that we have the effective feasible set
\begin{align}
    \widehat\X(A,b) &= \{\x\in[0, 1]^{9\times 9 \times 9} \mid A\x + \b = 0\}.
\end{align}
The constraint parameters $A$ and $\b$ are themselves parameterized such that at least one feasible point exists, which means that the optimization problem has a finite optimal solution, implying strong duality of~\eqref{eq:sudoku-saddlepoint}.

We follow the training protocol described by \citet{amos2017optnet} that minimizes the mean square error between predictions $\x^*(A, \b; \x_\text{inc})$ and one-hot encodings of the correctly solved Sudokus $\x_\text{true}$.
For evaluation, we follow \citet{amos2017optnet} in refining the predictions by taking an argmax over the one-hot dimension and report the percentage of violated ground-truth Sudoku constraints as the error. 

We modify the public codebase from \cite{amos2017optnet-repo}.
The dataset consists of $9000$ training and $1000$ test instances.
We choose the hyperparameters learning rate $\alpha$, $\tau$ and $\rho$ with a grid search. 
The best hyperparameters for \lpgd are $\tau=10^4$, $\rho=0.1$, $\alpha=0.1$, for gradient descent they are $\rho=10^{-3}$, $\alpha=0.1$. We use these hyperparameters in our evaluation.

Additional metrics are reported in \Figref{fig:sudoku-train-curves-all} and \Figref{fig:sudoku-train-test-error}. 
We observe that there is not significant difference between train and test metrics, which shows that our formulation of Sudoku \eqref{eq:sudoku-linear-program} allows generalizing across instances. 
The results also contain the upper and lower variations \lpgdl and \lpgdu in addition to the average \lpgda. 
We observe that \lpgda outperforms \lpgdu and \lpgdl, highlighting that both the lower and upper envelopes carry relevant information for the optimization.
This is intuitively understandable by considering that in \Figref{fig:envelope-visualization-rho} \lpgdl and \lpgdu provide non-zero gradients in different subsets of the domain, while \lpgda gives informative gradients in both subsets.

\begin{figure}[!htb]
    \centering
        \smallskip
        \smallskip
        \includegraphics[width=\linewidth]{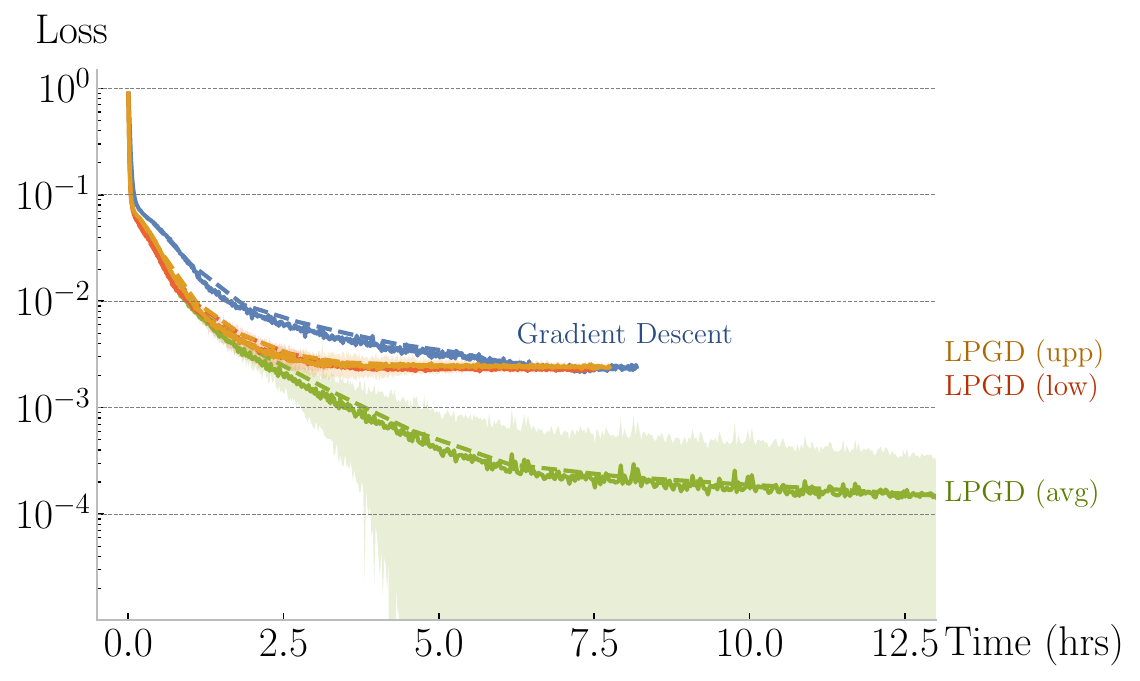}
        \smallskip
        \includegraphics[width=\linewidth]{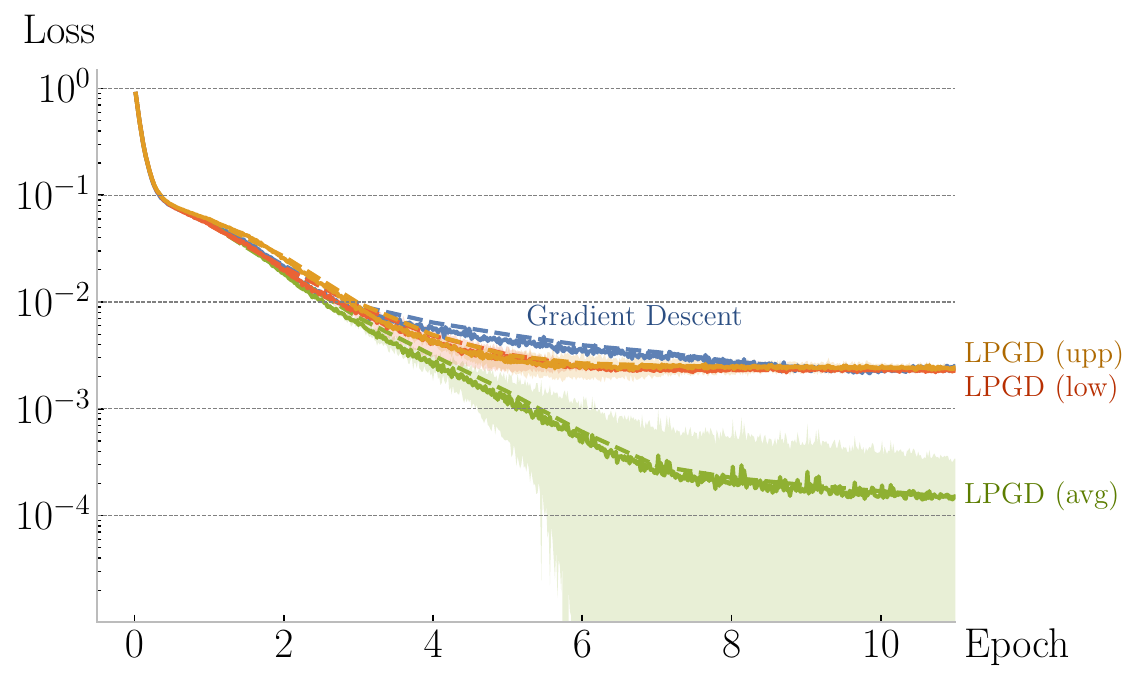}
        \smallskip
        \includegraphics[width=\linewidth]{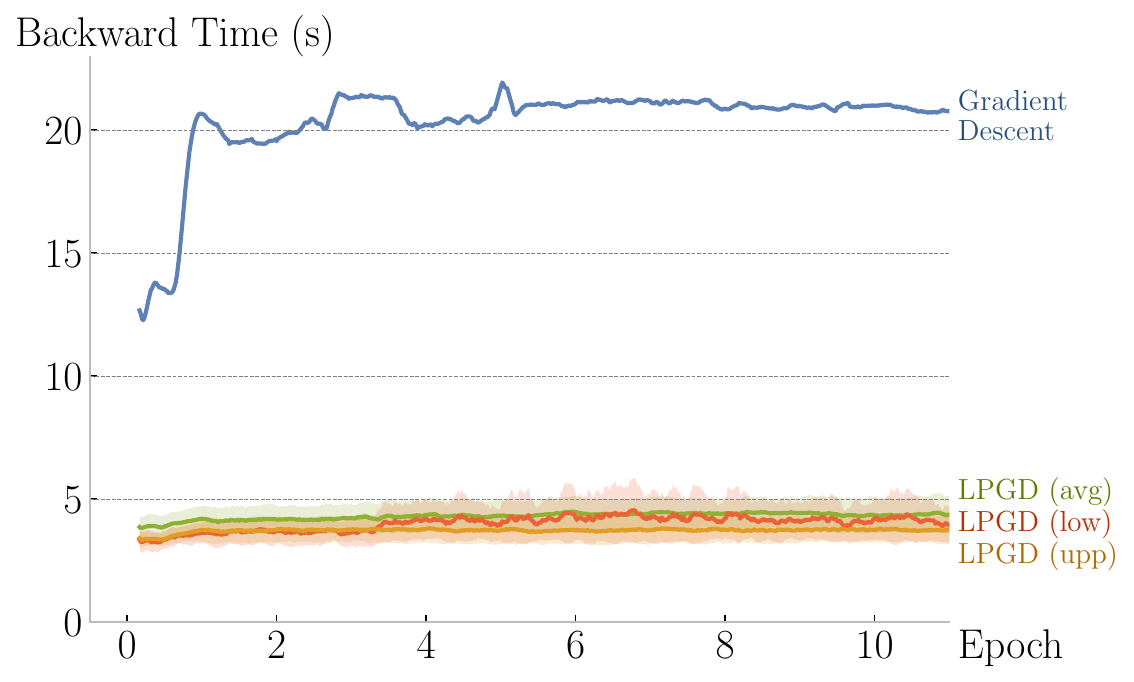}
        \smallskip
        \includegraphics[width=\linewidth]{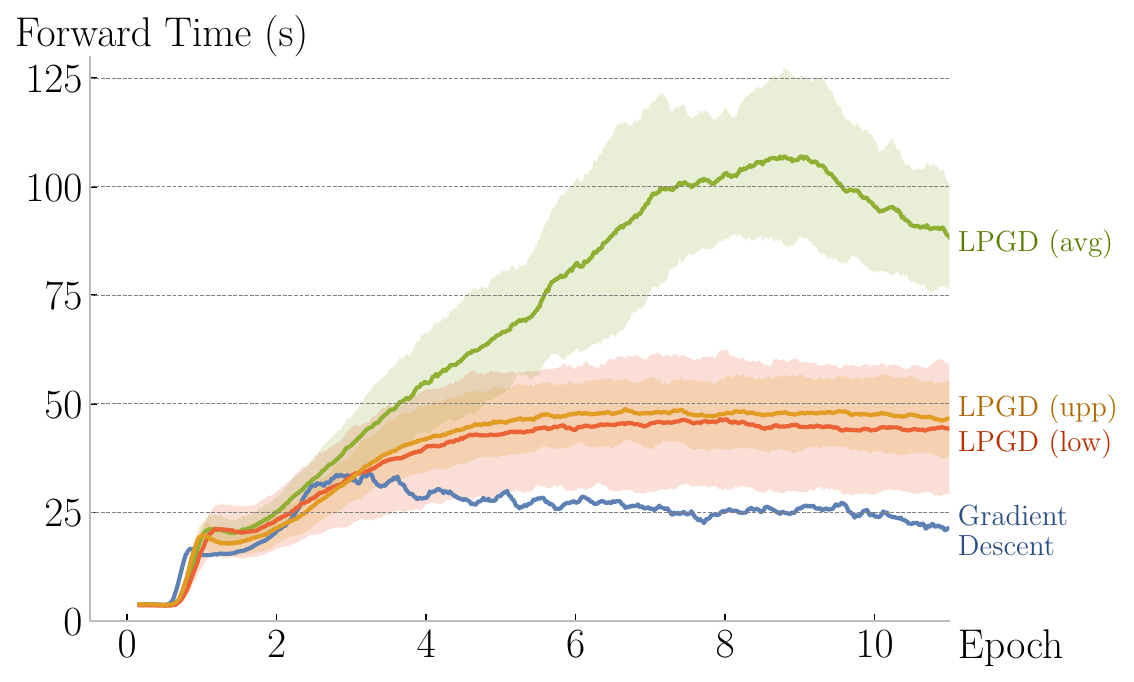}
    \caption{Comparison of \lpgd variations and gradient descent (GD) on the Sudoku experiment. 
    Reported are train and test MSE over epochs, wall-clock time, and time spent in the backward and forward passes. Dashed lines correspond to test data.
    Statistics are over $5$~restarts. 
    }
    \vspace{-2em}
    \label{fig:sudoku-train-curves-all}
\end{figure}
\begin{figure}
    \centering
    \includegraphics[width=0.93\linewidth]{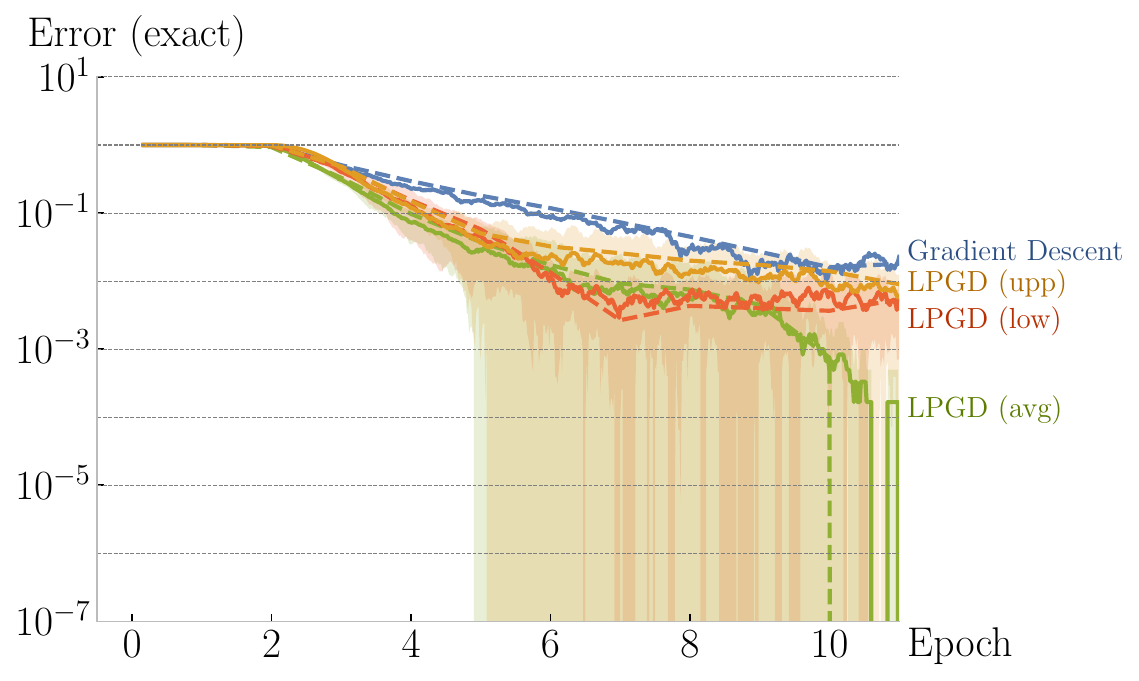}
    \includegraphics[width=0.93\linewidth]{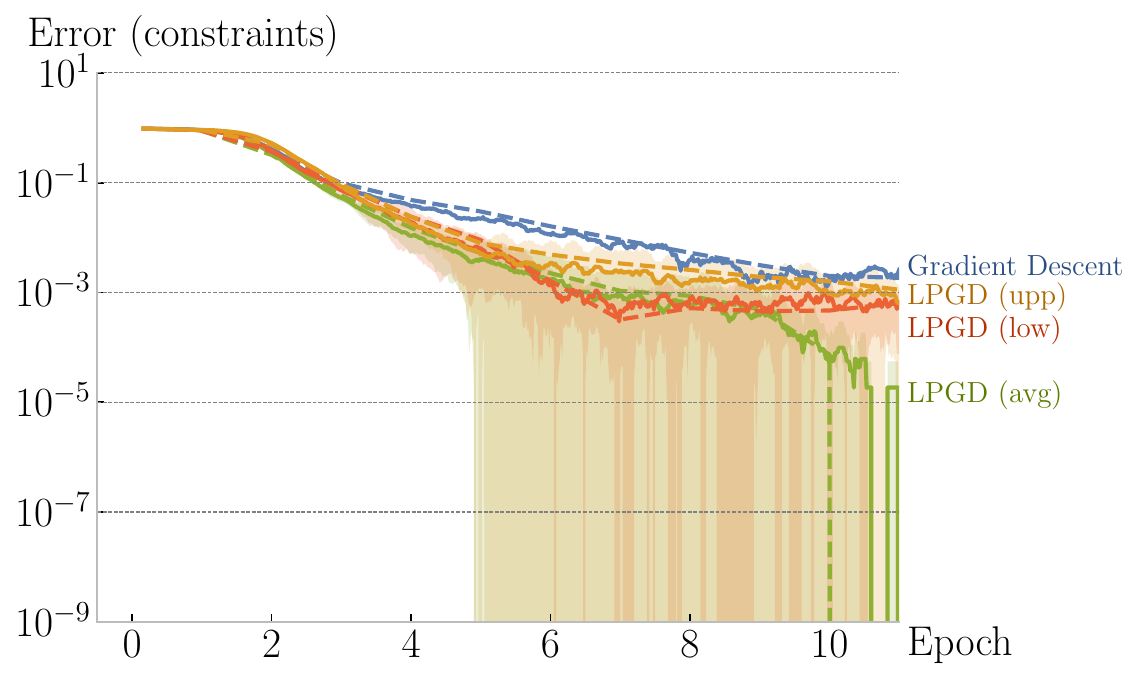}
    \caption{Comparison of \lpgd and gradient descent (GD) on the Sudoku experiment. Reported are train and test errors over epochs. Dashed lines correspond to test data.
    The exact error refers to the proportion of incorrect solutions (at least one violated Sudoku constraint), while the constraint error refers to the proportion of violated Sudoku constraints.
    Statistics are over $5$ restarts.
    \label{fig:sudoku-train-test-error}
    }
\end{figure}

\subsection{Tuning a Markowitz Control Policy}\label{sec:experimental-details-portfolio}
In the Markowitz Portfolio Optimization setting described in \citet[\S 5]{agrawal2020learning-control}, the task is to iteratively trade assets in a portfolio such that a utility function based on returns and risk is maximized over a trading horizon.
The trading policy is a convex optimization control policy, which determines trades by solving a parameterized convex optimization problem.
It maximizes a parameterized objective that trades off the expected return and the risk of the post-trade portfolio, subject to a constraint that ensures self-financing trades.
The parameters are initialized based on historical data, and differentiating through the optimization problem allows tuning them to maximize the utility on simulated evolutions of the asset values.
We refer the reader to \citet[\S 5]{agrawal2020learning-control} for a detailed description.

We conduct a sweep over the solver accuracy $\epsilon$, the results are reported in \Figref{fig:portfolio-2}.
The results show that \lpgd is more sensitive to inaccurate solutions than GD, as the update relies on accurate solutions for the finite difference.
Finally, we report statistics for $\tau=100$, $\alpha=0.001$, $\epsilon=0.0001$ over $3$ training seeds. We do not observe high variability in the performance.

\begin{figure}
    \centering
    \includegraphics[width=\linewidth]{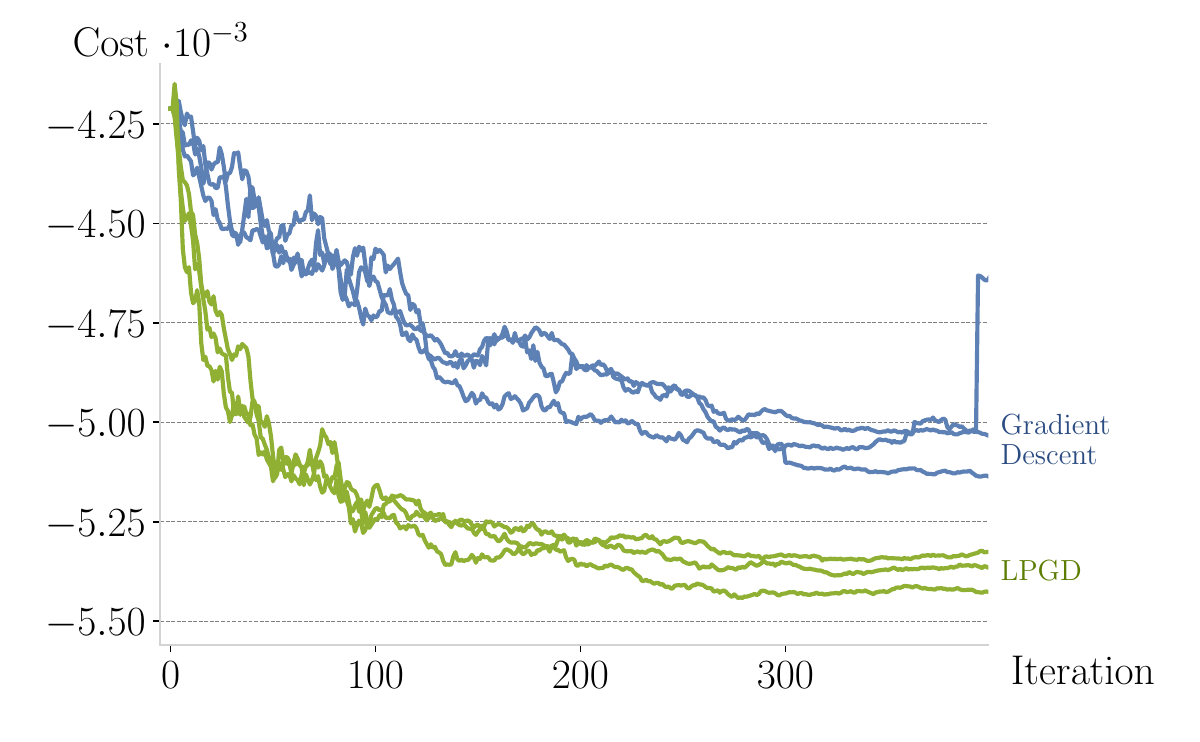}
    \caption{
    Comparison of \lpgda and GD in the Markowitz portfolio optimization experiment. 
    Reported is a sweep over training seeds, using $\alpha=0.001$, $\epsilon=0.0001$, $\tau=100$ and $\rho=0$.
    }
    \label{fig:portfolio-2}
\end{figure}

\section{List of Symbols}\label{sec:list-of-symbols}

\begin{tabular}{cl}
  \toprule
  $\x$ & primal variables \\
  $\s$ & optional slack variables \\
  $\y$ & dual variables \\
  $\z=(\x,\y)$ & primal \& dual variables \\
  $\z^*=(\x^*,\y^*)$ & optimal variables \eqref{eq:embedded-opt-problem}\\
  $\X, \Y$ & primal and dual feasible sets\\
  $\widetilde\X$ & primal effective feasible set \eqref{eq:effective-feasible-set}\\
  $\w$ & all parameters \\
  $\u$ & linear parameters \\
  $\v$ & non-linear parameters \\
  $\c$ & linear primal parameters \\
  $\b$ & linear dual parameters \\
  $\lag$ & Lagrangian \\
  $\lag^*$ & optimal Lagrangian \eqref{eq:optimal-lagrangian}\\
  $\Omega$ & non-linear part of Lagrangian \\
  \midrule
  $\tau$ & perturbation strength parameter \\
  $\env_{\tau\f}$ & Moreau envelope \eqref{eq:moreau-envelope}\\
  $\prox_{\tau\f}$ & Proximal map \eqref{eq:proximal-map}\\
  $\ell$ & loss on primal variables\\
  $D_\lag$ & Lagrangian difference \eqref{eq:lagrangian-difference}\\
  $D_\lag^*$ & Lagrangian divergence \eqref{eq:lagrangian-divergence}\\
  $\ell_\tau, \ell^\tau, \ell\mathmiddlescript\tau$ & $\lag$-envelopes \threeeqrefs{eq:lower-lag-moreau-envelope}{eq:upper-lag-moreau-envelope}{eq:eq:average-lag-proximal-map}\\
  $\z_{\tau}=(\x_{\tau},\y_\tau)$ & lower $\lag$-proximal map \eqref{eq:lower-lag-proximal-map}\\
  $\z^{\tau}=(\x^\tau,\y^\tau)$ & upper $\lag$-proximal map \eqref{eq:upper-lag-proximal-map}\\
  $\nabla\ell_{\tau}, \nabla\ell^{\tau}, \nabla\ell\mathmiddlescript{\tau}$ & \lppm updates \eqref{eq:lppm}\\
  $\Delta\ell_{\tau}, \Delta\ell^{\tau}, \Delta\ell\mathmiddlescript{\tau}$ & \lppm finite-differences \eqref{eq:finite-differences-lppm}\\
  \midrule
  $\widetilde\ell$ & linearization of $\ell$ at $\x^*$ \eqref{eq:linear-loss-approx}\\
  $\widetilde\ell_\tau, \widetilde\ell^\tau, \widetilde\ell\mathmiddlescript\tau$ & $\lag$-envelopes of lin. loss $\widetilde\ell$\\
  $\widetilde\z_{\tau}=(\widetilde\x_\tau,\widetilde\y_\tau)$ & lin. lower $\lag$-proximal map \eqref{eq:lower-lag-proximal-map-linearized}\\
  $\widetilde\z^{\tau}=(\widetilde\x^\tau,\widetilde\y^\tau)$ & lin. upper $\lag$-proximal map \eqref{eq:upper-lag-proximal-map-linearized}\\
  $\nabla\widetilde\ell_{\tau}, \nabla\widetilde\ell^{\tau}, \nabla\widetilde\ell\mathmiddlescript{\tau}$ & \lpgd updates \eqref{eq:lpgd}\\
  $\Delta\widetilde\ell_{\tau}, \Delta\widetilde\ell^{\tau}, \Delta\widetilde\ell\mathmiddlescript{\tau}$ & \lpgd finite-differences \\
  \midrule
  $\rho$ & augmentation strength parameter \\
  $\lag_\rho$ & augmented Lagrangian \eqref{eq:augmentation} \\
  $\ell_{\tau\rho}, \ell^{\tau\rho}, \ell\mathmiddlescript{\tau\rho}$ & aug. $\lag$-envelopes \eqref{eq:augmented-lower-lagrange-moreau-envelope}\\
  $\z_{\tau\rho}=(\x_{\tau\rho},\y_{\tau\rho})$ & aug. lower $\lag$-proximal map \eqref{eq:augmented-lower-lagrangian-proximal-map} \\
  $\z^{\tau\rho}=(\x^{\tau\rho},\y^{\tau\rho})$ & aug. upper $\lag$-proximal map \\
  $\nabla\ell_{\tau\rho}, \nabla\ell^{\tau\rho}, \nabla\ell\mathmiddlescript{\tau\rho}$ & aug. \lppm updates \eqref{eq:augmented-lower-lppm-lpgd} \\
  $\Delta\ell_{\tau\rho}, \Delta\ell^{\tau\rho}, \Delta\ell\mathmiddlescript{\tau\rho}$ & aug. \lppm finite-differences \eqref{eq:finite-differences-augmented} \\
  \midrule
  $\widetilde\ell_{\tau\rho}, \widetilde\ell^{\tau\rho}, \widetilde\ell\mathmiddlescript{\tau\rho}$ & aug.  $\lag$-envelopes of lin. loss $\widetilde\ell$ \\
  $\widetilde\z_{\tau\rho}=(\widetilde\x_{\tau\rho},\widetilde\y_{\tau\rho})$ & aug. lin. lower $\lag$-proximal map \\
  $\widetilde\z^{\tau\rho}=(\widetilde\x^{\tau\rho},\widetilde\y^{\tau\rho})$ & aug. lin. upper $\lag$-proximal map \\
  $\nabla\widetilde\ell_{\tau\rho}, \nabla\widetilde\ell^{\tau\rho}, \nabla\widetilde\ell\mathmiddlescript{\tau\rho}$ & aug. lin. \lpgd updates \eqref{eq:augmented-lower-lppm-lpgd} \\
  $\Delta\widetilde\ell_{\tau\rho}, \Delta\widetilde\ell^{\tau\rho}, \Delta\widetilde\ell\mathmiddlescript{\tau\rho}$ & aug. lin. \lpgd \\&finite-differences \eqref{eq:finite-differences-augmented} \\
  \midrule
  $\L$ & loss on primal \& dual variables \\
  $\widetilde\L$ & linearization of $\L$ at $\z^*$ \\
  \bottomrule
\end{tabular}

\end{document}